\documentclass[10pt,twocolumn]{article}
\usepackage[letterpaper,textwidth=6.75in,textheight=9.25in,columnsep=0.25in]{geometry}

\usepackage[round]{natbib}
\usepackage{amsmath,amssymb,amsthm,mathtools}
\usepackage{booktabs,multirow}
\usepackage{graphicx}

\usepackage{xcolor}
\usepackage[colorlinks=true,linkcolor=black,citecolor=black,urlcolor=black]{hyperref}
\expandafter\def\expandafter\UrlBreaks\expandafter{\UrlBreaks\do\a\do\b\do\c\do\d\do\e\do\f\do\g\do\h\do\i\do\j\do\k\do\l\do\m\do\n\do\o\do\p\do\q\do\r\do\s\do\t\do\u\do\v\do\w\do\x\do\y\do\z\do\A\do\B\do\C\do\D\do\E\do\F\do\G\do\H\do\I\do\J\do\K\do\L\do\M\do\N\do\O\do\P\do\Q\do\R\do\S\do\T\do\U\do\V\do\W\do\X\do\Y\do\Z\do\0\do\1\do\2\do\3\do\4\do\5\do\6\do\7\do\8\do\9}

\newtheorem{theorem}{Theorem}
\newtheorem{proposition}{Proposition}
\newtheorem{definition}{Definition}
\newtheorem{lemma}{Lemma}
\newtheorem{remark}{Remark}

\newcommand{\Pt}{\tilde P}
\newcommand{\E}{\mathbb{E}}

\newcommand{\topT}{\mathrm{top}_T}
\newcommand{\Ind}{\mathbf 1}
\newcommand{\Y}{\mathcal Y}
\newcommand{\X}{\mathcal X}

\newcommand{\Nbcpealpha}{10.0}
\newcommand{\Nbcpmeansize}{1.66}
\newcommand{\Nbcpmisexact}{8.9}
\newcommand{\Nbcpmisnb}{5.9}
\newcommand{\Nbcpshareover}{50}
\newcommand{\Ncifaucncells}{30}

\newcommand{\Nciffailauc}{0.941}

\newcommand{\NcifidncellsTtwothreenonwall}{42}
\newcommand{\NcifidncellsTtwothreewall}{18}
\newcommand{\NcifidvalidTtwothreenonwall}{39}
\newcommand{\NcifidvalidTtwothreewall}{9}
\newcommand{\Ncifmaeatcbeatsscore}{12}

\newcommand{\Ncifmaengroups}{12}

\newcommand{\NcifnewTOnenonwallncells}{9}

\newcommand{\NcifnewTTwononwallWSCPvalid}{4}

\newcommand{\NcifnewTTwononwallncells}{9}

\newcommand{\NcifnewnonwallrelbiasSCORE}{$-$3}

\newcommand{\Ncifnsettings}{90}

\newcommand{\NcifoldTOnedegenerate}{33}
\newcommand{\NcifoldTOnedraws}{900}

\newcommand{\NcifoldTTwononwallWSCPvalid}{1}

\newcommand{\NcifoldTTwononwallncells}{12}

\newcommand{\NcifoldnonwallrelbiasSCORE}{$-$21}

\newcommand{\Ncifsignedatcunder}{73}

\newcommand{\Ncifsignedembunder}{79}

\newcommand{\Ncifsignednoneunder}{90}

\newcommand{\Ncifsignedscoreunder}{62}
\newcommand{\NciftightbootEMB}{4}
\newcommand{\NciftightbootSCORE}{50}
\newcommand{\NciftightcloserEMB}{71}

\newcommand{\Nciftightunder}{62}

\newcommand{\NciftotTOnenonwallSCOREidentityvalid}{9}

\newcommand{\NciftotTOnenonwallncells}{21}

\newcommand{\Ncifwallembaucmin}{0.960}
\newcommand{\Ncifwallrelmax}{87}
\newcommand{\Ncifwallrelmin}{21}
\newcommand{\Ncifwallscoreaucmax}{0.972}
\newcommand{\Ncifwallscoreaucmin}{0.527}
\newcommand{\NdbTOnemaxse}{0.012}
\newcommand{\NdbTOneoraclecorrmaxabs}{0.020}

\newcommand{\NdbTOneunweightedtruth}{0.087}
\newcommand{\NdbTTwomaxse}{0.012}

\newcommand{\Ndbatcbelowmiss}{39}
\newcommand{\Ndbatcfarther}{12}

\newcommand{\NdbbenignTwoTOnetruth}{0.036}

\newcommand{\Ndbbenignscorecorrmaxabs}{0.002}
\newcommand{\Ndbcalmaxse}{0.010}
\newcommand{\Ndbcalmselower}{24}
\newcommand{\Ndbcalmsetotal}{24}

\newcommand{\Ndbclipfeatmaxchange}{0.003}

\newcommand{\Ndbeasybenignmaxabs}{0.046}
\newcommand{\NdbhardOnemsecalmax}{2.81}
\newcommand{\NdbhardOnemsecalmin}{2.19}
\newcommand{\NdbhardOnemsegb}{3.17}

\newcommand{\NdbhardnessOneTOneATC}{+0.211}
\newcommand{\NdbhardnessOneTOneATCMC}{+0.179}
\newcommand{\NdbhardnessOneTOneATCTS}{+0.201}
\newcommand{\NdbhardnessOneTOneATCcode}{+0.206}
\newcommand{\NdbhardnessOneTOneGBcorrected}{$-$0.081}

\newcommand{\NdbhardnessOneTOneSCOREScorrected}{+0.042}

\newcommand{\NdbhardnessTwoTOneATC}{+0.461}
\newcommand{\NdbhardnessTwoTOneATCMC}{+0.409}
\newcommand{\NdbhardnessTwoTOneATCTS}{+0.445}
\newcommand{\NdbhardnessTwoTOneATCcode}{+0.458}
\newcommand{\NdbhardnessTwoTOneGBcorrected}{$-$0.101}

\newcommand{\NdbhardnessTwoTOneSCOREScorrected}{$-$0.006}

\newcommand{\Ndbtruthhalfmaxdev}{0.004}

\newcommand{\NdbwscpGBhardTwoexcess}{+0.164}
\newcommand{\Ndbwscporaclemaxexcess}{+0.002}

\newcommand{\Nnataucmax}{0.66}
\newcommand{\Nnatmaxse}{0.64}
\newcommand{\Nnatrelmatch}{0.34}
\newcommand{\Nnddbmaxchange}{0.001}
\newcommand{\Nndfailmissmin}{28}
\newcommand{\Nndlimitmaxshort}{15.0}
\newcommand{\Nndscoremaxchange}{0.02}
\newcommand{\Nnoshclipalphashift}{2.3}
\newcommand{\Nnoshcoveredmintwothree}{50}
\newcommand{\NnoshnewTOnecovered}{16}

\newcommand{\NnoshnewToneidbelow}{1.7}
\newcommand{\NnosholdTOnecovered}{10}

\newcommand{\Nnoshtightmaxabs}{0.1}
\newcommand{\Nreldivcorrmax}{0.96}
\newcommand{\Nreldivcorrmin}{0.92}
\newcommand{\Nreldivcorrnwmax}{0.91}
\newcommand{\Nreldivcorrnwmin}{0.83}

\newcommand{\NrelnewBTwoZerocorr}{0.969}

\newcommand{\Nrelnewbelow}{45}

\newcommand{\Nrelnewmeandiffpp}{0.8}
\newcommand{\Nrelnewnonwallcorr}{0.42}
\newcommand{\Nrelnewnonwallrelsize}{0.21}
\newcommand{\Nrelnewrelcorr}{0.95}
\newcommand{\Nrelnewrelcorrnonwall}{0.87}

\newcommand{\Nrelnewwallrelsize}{0.51}

\newcommand{\NreloldBTwoZerocorr}{0.995}

\newcommand{\Nreloldbelow}{44}

\newcommand{\Nreloldmeandiffpp}{2.4}
\newcommand{\Nreloldnonwallcorr}{0.97}
\newcommand{\Nreloldnonwallrelsize}{0.38}
\newcommand{\Nreloldnsettings}{45}
\newcommand{\Nreloldrelcorr}{0.94}
\newcommand{\Nreloldrelcorrnonwall}{0.88}

\newcommand{\Nreloldwallrelsize}{0.73}
\newcommand{\Nrobqabovemax}{0.27}

\newcommand{\Ntiesold}{70}

\newcommand{\Ntiespairsold}{480{,}000}
\newcommand{\Nunclipbound}{10}
\newcommand{\Nuncliptightrowmax}{0.02}
\newcommand{\Nwallexauc}{0.96}
\newcommand{\Nwallexjpegfive}{14.7}
\newcommand{\Nwallexjpegthree}{12.1}
\newcommand{\Nwallextopone}{7.2}
\newcommand{\NwscpTTwooverhi}{27}
\newcommand{\NwscpTTwooverlo}{18}

\newcommand{\NwscpnewTTworatiomedian}{0.35}

\newcommand{\NwscpoldTTworatiomedian}{0.72}

\newcommand{\StmtDefSetting}{%
Let $K\ge 2$ and $\mathcal Y=\{1,\dots,K\}$. Let $P$ and $\Pt$ be probability distributions on
$\mathcal X\times\mathcal Y$. Let $(X_i,Y_i)$, $i=1,\dots,n$, be i.i.d.\ from $P$, and let $(X_{n+1},Y_{n+1})\sim\Pt$
be independent of them. Assume $\Pt(\mathrm dx,\mathrm dy)=w(x)\,P(\mathrm dx,\mathrm dy)$ for a measurable
$w:\mathcal X\to(0,\infty)$. Let $S:\mathcal X\times\mathcal Y\to[0,\infty)$ be a measurable nonconformity score fixed independently
of the data, and let $T\in\{1,\dots,K-1\}$. We define $w_j:=w(X_j)$ and $S_j:=S(X_j,Y_j)$ for $j\le n+1$, and
$W:=\sum_{j=1}^{n+1}w_j$. We use the conventions $0/0:=0$ and $1/0:=\infty$.}

\newcommand{\StmtDefEvalue}{%
We define, for $y\in\mathcal Y$,
\[
E^{(w)}(y):=\frac{W\,S(X_{n+1},y)}{\sum_{i=1}^n w_iS_i+w_{n+1}S(X_{n+1},y)}.
\]
For $e\in[0,\infty)^{\mathcal Y}$ we define $F(e):=\inf\{\alpha\in(0,1):\#\{y:e_y<1/\alpha\}\le T\}$ with
$\inf\emptyset:=1$. We define $\tilde\alpha:=F(E^{(w)})$ and $\hat C:=\{y:E^{(w)}(y)<1/\tilde\alpha\}$.}

\newcommand{\StmtLemF}{%
Let $e,e'\in[0,\infty)^{\mathcal Y}$, and let $e_{(T+1)}$ be the $(T+1)$-th smallest entry of $e$, with $1/0:=\infty$.
Then $F(e)=\min\{1,1/e_{(T+1)}\}$ and $|F(e)-F(e')|\le\min\{1,\max_y|e_y-e'_y|\}$.}

\newcommand{\StmtThmValidity}{%
Let the setting of Definition~\ref{def:setting} hold. Then $\E_{\Pt}\bigl[E^{(w)}(Y_{n+1})\bigr]\le1$.}

\newcommand{\StmtThmPosthoc}{%
Let the setting of Definition~\ref{def:setting} hold, and let $\hat\alpha$ be a random variable with values in $(0,1]$.
Then $\E_{\Pt}\bigl[\Ind\{E^{(w)}(Y_{n+1})\ge 1/\hat\alpha\}/\hat\alpha\bigr]\le 1$ and
$\E_{\Pt}\bigl[\Ind\{Y_{n+1}\notin\hat C\}/\tilde\alpha\bigr]\le 1$.}

\newcommand{\StmtThmCoverage}{%
Let the setting of Definition~\ref{def:setting} hold. We define $I:=\Ind\{Y_{n+1}\notin\hat C\}$. Assume
$\mathrm{Cov}_{\Pt}(\tilde\alpha,I/\tilde\alpha)\le 0$. Then $\Pt(Y_{n+1}\notin\hat C)\le\E_{\Pt}[\tilde\alpha]$.}

\newcommand{\StmtRemCov}{%
Let $M$ be a version of $a\mapsto\Pt(Y_{n+1}\notin\hat C\mid\tilde\alpha=a)$. The condition of
Theorem~\ref{thm:coverage} is equivalent to $\mathrm{Cov}_{\Pt}(\tilde\alpha,M(\tilde\alpha)/\tilde\alpha)\le0$, and
holds when $a\mapsto M(a)/a$ is non-increasing on the support of $\tilde\alpha$. Monotonicity is sufficient, not
necessary: with $\tilde\alpha$ uniform on $\{1/4,1/2,1\}$ and $M=(1/2,0,1/2)$ the covariance is $-11/72$.}

\newcommand{\StmtPropSqrt}{%
Let the setting of Definition~\ref{def:setting} hold. Then $\Pt(Y_{n+1}\notin\hat C)\le\bigl(\E_{\Pt}[\tilde\alpha]\bigr)^{1/2}$.}

\newcommand{\StmtDefTransform}{%
Let the setting of Definition~\ref{def:setting} hold. We define $\tau(x)$ as the $(T+1)$-th smallest of
$S(x,1),\dots,S(x,K)$, $\topT(x):=\{y:S(x,y)<\tau(x)\}$ (so $|\topT(x)|\le T$, with equality when the scores at $x$
are distinct), $m(x,y):=\Ind\{y\notin\topT(x)\}$ and $h(x,y):=\tau(x)\,m(x,y)$. We define $h_j:=h(X_j,Y_j)$,
$m_j:=m(X_j,Y_j)$ and
\[
E^{(w,h)}(y):=\frac{W\,h(X_{n+1},y)}{\sum_{i=1}^nw_ih_i+w_{n+1}h(X_{n+1},y)},
\]
$\tilde\alpha_h:=F(E^{(w,h)})$ and $\hat C_h:=\{y:E^{(w,h)}(y)<1/\tilde\alpha_h\}$.}

\newcommand{\StmtThmTransform}{%
Let the setting of Definition~\ref{def:transform} hold. Then $\E_{\Pt}\bigl[E^{(w,h)}(Y_{n+1})\bigr]\le1$.}

\newcommand{\StmtLemMarkov}{%
Let $e\ge0$ and $\hat\alpha\in(0,1]$ be random variables. Then $\mathbb P(e\ge1/\hat\alpha)\le\E[\hat\alpha\,e]$.}

\newcommand{\StmtThmCorrected}{%
Let the setting of Definition~\ref{def:transform} hold. Then
$\Pt(Y_{n+1}\notin\hat C_h)\le\E_{\Pt}\bigl[\tilde\alpha_hE^{(w,h)}(Y_{n+1})\bigr]$.}

\newcommand{\StmtPropTwopoint}{%
Let the setting of Definition~\ref{def:transform} hold. We define
$c:=W\tau(X_{n+1})/\bigl(\sum_{i=1}^nw_ih_i+w_{n+1}\tau(X_{n+1})\bigr)$ $(0/0:=0)$. Then $\tilde\alpha_h=\min\{1,1/c\}$
and $\tilde\alpha_hE^{(w,h)}(Y_{n+1})=\min\{1,c\}\,m_{n+1}$. Assume moreover $\Pt(c<1,\,m_{n+1}=1)=0$. Then
$\E_{\Pt}\bigl[\tilde\alpha_hE^{(w,h)}(Y_{n+1})\bigr]=\Pt\bigl(Y_{n+1}\notin\topT(X_{n+1})\bigr)=\E_P[w(X)\,m(X,Y)]$.}

\newcommand{\StmtDefReports}{%
Let the setting of Definition~\ref{def:transform} hold, and let $\omega_1,\dots,\omega_n>0$ be weights on the
calibration points. For $j\le n$ we define (with $0/0:=0$)
\[
E^{(\omega),j}(y):=\frac{(\sum_{i=1}^n\omega_i)\,S(X_j,y)}{\sum_{i\ne j}\omega_iS_i+\omega_jS(X_j,y)},
\]
\[
E^{(\omega,h),j}(y):=\frac{(\sum_{i=1}^n\omega_i)\,h(X_j,y)}{\sum_{i\ne j}\omega_ih_i+\omega_jh(X_j,y)},
\]
$\tilde\alpha_j(\omega):=F(E^{(\omega),j})$, $\tilde\alpha_{h,j}(\omega):=F(E^{(\omega,h),j})$ and
$\tilde b_j(\omega):=\tilde\alpha_{h,j}(\omega)E^{(\omega,h),j}(Y_j)$. We define the identity, transformed and tight
reports
\[
\hat\alpha^{\rm id}(\omega):=\frac{\sum_j\omega_j\tilde\alpha_j(\omega)}{\sum_j\omega_j},\quad
\hat\alpha^{\rm tr}(\omega):=\frac{\sum_j\omega_j\tilde\alpha_{h,j}(\omega)}{\sum_j\omega_j},
\]
\[
\hat b(\omega):=\frac{\sum_j\omega_j\tilde b_j(\omega)}{\sum_j\omega_j}.
\]}

\newcommand{\StmtRemNondeg}{%
Let the setting of Definition~\ref{def:reports} hold. We define
$\bar h(\omega):=\sum_{i=1}^n\omega_ih_i/\sum_{i=1}^n\omega_i$ and, for $j\le n$,
$c_j(\omega):=(\sum_{i=1}^n\omega_i)\tau(X_j)/(\sum_{i\ne j}\omega_ih_i+\omega_j\tau(X_j))$. Then
$c_j(\omega)=\tau(X_j)/\bar h(\omega)$ for every $j\le n$ with $m_j=1$,
\[
\hat b(\omega)=\frac{\sum_{j:\,m_j=1}\omega_j\min\{1,\tau(X_j)/\bar h(\omega)\}}{\sum_{j=1}^n\omega_j}
\le\frac{\sum_{j=1}^n\omega_jm_j}{\sum_{j=1}^n\omega_j},
\]
the inequality is an equality if and only if $\tilde b_j(\omega)=m_j$ for every $j\le n$, and
$\tilde b_j(\omega)=m_j$ for every $j\le n$ if and only if $\tau(X_j)>0$ and $\tau(X_j)\ge\bar h(\omega)$ for every
$j\le n$ with $m_j=1$. Let $\tau_{\min}>0$ be a constant, and assume moreover $\tau(x)\ge\tau_{\min}$ for every
$x\in\mathcal X$. Then $\tilde b_j(\omega)=m_j$ for every $j\le n$ on the event $\{\bar h(\omega)\le\tau_{\min}\}$.
Assume moreover $S\le S_{\max}<\infty$. Let $\hat\omega:\mathcal X\to(0,\omega_{\max}]$ be a measurable function that is
fixed, or fitted on data independent of $(X_i,Y_i)_{i\le n+1}$; in the latter case all probabilities are conditional on
those data. Let $\omega_i:=\hat\omega(X_i)$. We define $\tilde\omega:=\hat\omega/\int\hat\omega\,\mathrm dP_X$. Assume
$\E_P[\tilde\omega(X)h(X,Y)]<\tau_{\min}$. Then $\mathbb P\bigl(\tilde b_j(\omega)=m_j\ \forall j\le n\bigr)\to1$ as
$n\to\infty$.}

\newcommand{\StmtThmLoo}{%
Let the setting of Definition~\ref{def:reports} hold with $\omega_i=w(X_i)$. Assume $w\le w_{\max}<\infty$,
$S\le S_{\max}<\infty$ and $D^\star:=\E_P[w(X)S(X,Y)]>0$. Then
$\hat\alpha^{\rm id}(w)-\E_{\Pt}[\tilde\alpha]=O_P(n^{-1/2})$ as $n\to\infty$.}

\newcommand{\StmtThmCorrloo}{%
Let the setting of Definition~\ref{def:reports} hold with $\omega_i=w(X_i)$. Assume $w\le w_{\max}<\infty$,
$S\le S_{\max}<\infty$ and $D^\star_h:=\E_P[w(X)h(X,Y)]>0$. We define
$\beta^\star:=\E_{\Pt}[\tilde\alpha_hE^{(w,h)}(Y_{n+1})]$. Then $\hat b(w)-\beta^\star=O_P(n^{-1/2})$ as $n\to\infty$.}

\newcommand{\StmtPropFs}{%
Let the setting of Definition~\ref{def:reports} hold, and let $\hat\omega_1,\dots,\hat\omega_n>0$ and
$\omega_1,\dots,\omega_n>0$ be two weight vectors on the calibration points. Assume $\tilde b_j(\hat\omega)=m_j$ and
$\tilde b_j(\omega)=m_j$ for every $j\le n$, that is, non-degeneracy of both weight vectors. We define $\bar{\hat\omega}:=\frac1n\sum_j\hat\omega_j$ and
$\bar\omega:=\frac1n\sum_j\omega_j$. Then
\[
\hat b(\hat\omega)-\hat b(\omega)=\frac1n\sum_{j=1}^n\Bigl(\frac{\hat\omega_j}{\bar{\hat\omega}}-\frac{\omega_j}{\bar\omega}\Bigr)m_j .
\]}

\newcommand{\StmtPropWest}{%
Let the setting of Definition~\ref{def:reports} hold. Let $\hat\omega:\mathcal X\to(0,\omega_{\max}]$ be a measurable function
that is fixed, or fitted on data independent of $(X_i,Y_i)_{i\le n+1}$; in the latter case all probabilities are
conditional on those data. Let $\omega_i:=\hat\omega(X_i)$. Assume
$\mathbb P\bigl(\tilde b_j(\omega)=m_j\ \forall j\le n\bigr)\to1$ as $n\to\infty$. We define
$\tilde\omega:=\hat\omega/\int\hat\omega\,\mathrm dP_X$. Then, as $n\to\infty$, $\hat b(\omega)\to\E_P[\tilde\omega(X)m(X,Y)]$
in probability and
$\hat b(\omega)-\Pt\bigl(Y_{n+1}\notin\topT(X_{n+1})\bigr)\to\E_P\bigl[(\tilde\omega(X)-w(X))\,m(X,Y)\bigr]$ in
probability.}

\newcommand{\StmtPropDecomp}{%
Let the setting of Definition~\ref{def:reports} hold. Let $\hat\omega:\mathcal X\to(0,\omega_{\max}]$ be a measurable
function that is fixed, or fitted on data independent of $(X_i,Y_i)_{i\le n+1}$; in the latter case all probabilities
are conditional on those data. Let $\omega_i:=\hat\omega(X_i)$, and assume $\mathbb P\bigl(\tilde b_j(\omega)=m_j\ \forall
j\le n\bigr)\to1$ as $n\to\infty$. We define $\tilde\omega:=\hat\omega/\int\hat\omega\,\mathrm dP_X$. Let $\phi$ be a
measurable map from $\mathcal X$ to a measurable space $\mathcal Z$. We define $v_\phi$, $\mu_\phi$
and $\tilde\mu_\phi$ as versions of $z\mapsto\E_P[w(X)\mid\phi(X)=z]$, $z\mapsto P(m(X,Y)=1\mid\phi(X)=z)$ and
$z\mapsto\Pt(m(X,Y)=1\mid\phi(X)=z)$, the last two with values in $[0,1]$. Then, as $n\to\infty$,
\begin{align*}
&\hat b(\omega)-\Pt\bigl(Y_{n+1}\notin\topT(X_{n+1})\bigr)\\
&\quad\to\E_P\bigl[(\tilde\omega(X)-v_\phi(\phi(X)))\,m(X,Y)\bigr]\\
&\qquad-\E_{\Pt}\bigl[\tilde\mu_\phi(\phi(X))-\mu_\phi(\phi(X))\bigr]
\end{align*}
in probability.}

\newcommand{\StmtPropCompare}{%
Let the setting of Definition~\ref{def:transform} hold. Let $\phi:\mathcal X\to\mathcal Z$ and
$\psi:\mathcal X\to\mathcal Z'$ be measurable maps into measurable spaces with $\psi=f\circ\phi$ for a measurable
$f:\mathcal Z\to\mathcal Z'$. We define $v_\phi$, $\mu_\phi$ and $\tilde\mu_\phi$ as in Proposition~\ref{prop:decomp},
$M_\phi:=\E_P[m(X,Y)\mid\phi(X)]$ and $B_\phi:=\E_P[(v_\phi(\phi(X))-w(X))\,m(X,Y)]$, and $v_\psi$, $M_\psi$ and $B_\psi$ in the
same way with $\psi$ in place of $\phi$. Then
$B_\phi=-\E_{\Pt}\bigl[\tilde\mu_\phi(\phi(X))-\mu_\phi(\phi(X))\bigr]$ and
\[
B_\psi=B_\phi-\E_P\bigl[v_\phi(\phi(X))M_\phi-\E_P[v_\phi(\phi(X))\mid\psi(X)]\,\E_P[M_\phi\mid\psi(X)]\bigr].
\]
Assume moreover that $w(X)$ is $\sigma(\phi(X))$-measurable. Then $B_\phi=0$.}

\newcommand{\StmtDefGr}{%
Let the setting of Definition~\ref{def:transform} hold. We define $G(x)\in[0,\infty)^K$ as the vector
$(S(x,y))_{y\in\mathcal Y}$ sorted in increasing order, and $R(x,y)\in\{1,\dots,K\}$ as the position of label $y$ when
$\mathcal Y$ is sorted by increasing $S(x,\cdot)$ with ties broken by label index, so that $y\mapsto R(x,y)$ is a
bijection and $S(x,y)=G(x)_{R(x,y)}$ for every $(x,y)$. We define $v$ as a version of
$g\mapsto\E_P[w(X)\mid G(X)=g]$ with $v>0$.}

\newcommand{\StmtPropReduction}{%
Let the setting of Definition~\ref{def:GR} hold. Assume that there are measurable $q_r:[0,\infty)^K\to[0,1]$,
$r=1,\dots,K$, with $P(R(X,Y)=r\mid G(X))=q_r(G(X))$ $P$-a.s.\ and $\Pt(R(X,Y)=r\mid G(X))=q_r(G(X))$ $\Pt$-a.s.\ for
every $r$. We define $E^{(v\circ G)}$, $\tilde\alpha^{(v\circ G)}$, $\hat C^{(v\circ G)}$, $E^{(v\circ G,h)}$,
$\tilde\alpha^{(v\circ G)}_h$ and $\hat C^{(v\circ G)}_h$ as in Definitions~\ref{def:evalue} and~\ref{def:transform}
with $w(X_j)$ replaced by $v(G(X_j))$ for $j=1,\dots,n+1$ (so $W$ is replaced by $\sum_{j=1}^{n+1}v(G(X_j))$).
Then (i) $\E_{\Pt}[\eta(G(X),R(X,Y))]=\E_P[v(G(X))\eta(G(X),R(X,Y))]$ for every bounded measurable $\eta$, and
(ii) $\E_{\Pt}[E^{(v\circ G)}(Y_{n+1})]\le1$,
$\E_{\Pt}[\Ind\{Y_{n+1}\notin\hat C^{(v\circ G)}\}/\tilde\alpha^{(v\circ G)}]\le1$,
$\Pt(Y_{n+1}\notin\hat C^{(v\circ G)})\le(\E_{\Pt}[\tilde\alpha^{(v\circ G)}])^{1/2}$,
$\E_{\Pt}[E^{(v\circ G,h)}(Y_{n+1})]\le1$ and
$\Pt(Y_{n+1}\notin\hat C^{(v\circ G)}_h)\le\E_{\Pt}[\tilde\alpha^{(v\circ G)}_hE^{(v\circ G,h)}(Y_{n+1})]$.}

\newcommand{\StmtPropSufficient}{%
Let the setting of Definition~\ref{def:GR} hold. Assume that $w(X)=v(G(X))$ $P$-a.s., or that
$P(R(X,Y)=r\mid X)=P(R(X,Y)=r\mid G(X))$ $P$-a.s.\ for every $r$, or that
$\E_{\Pt}[\eta(G(X),R(X,Y))]=\E_P[v(G(X))\eta(G(X),R(X,Y))]$ for every bounded measurable $\eta$. Then there are measurable
$q_r:[0,\infty)^K\to[0,1]$, $r=1,\dots,K$, with $P(R(X,Y)=r\mid G(X))=q_r(G(X))$ $P$-a.s.\ and
$\Pt(R(X,Y)=r\mid G(X))=q_r(G(X))$ $\Pt$-a.s.\ for every $r$.}

\title{Backward Conformal Prediction under Covariate Shift}
\author{%
  Amin Meghrazi \\
  The Ohio State University \\
  \texttt{meghrazi.1@osu.edu}
  \and
  Yuanyuan Li \\
  Munich RE \\
  \texttt{yli@munichre.com}
}
\date{}

\renewenvironment{abstract}{\begin{center}\large\bfseries\abstractname\end{center}}{\vspace{1ex}}
\begin{document}
\maketitle

\begin{abstract}
Backward conformal prediction (BCP) fixes the size of a prediction set and estimates its miscoverage, but assumes calibration and test data are exchangeable. We extend BCP to covariate shift through importance weighting, preserving its guarantees when the density ratio is known. We show that BCP's tight miscoverage report depends only on the relative weight assigned to missed calibration examples, explaining why a better domain classifier can sometimes produce a worse estimate. Motivated by the fact that BCP depends only on nonconformity scores, we propose estimating density ratios on sorted score vectors rather than high-dimensional inputs. We prove that this approach recovers the same report asymptotically under an interpretable condition and characterize the resulting bias otherwise. Experiments with simulated and benchmark distribution shifts show that the proposed method improves BCP's miscoverage estimates, while score-based weighting consistently yields more accurate estimates than embedding-based weighting.
\end{abstract}
\section{Introduction}\label{sec:intro}

Conformal prediction provides label sets with guaranteed coverage under exchangeable calibration and test data, where the true label is included with probability at least $1-\alpha$ for the predetermined $\alpha$ \citep{vovk2005algorithmic}. The set size depends on the data. Backward conformal prediction (BCP) inverts this approach \citep{gauthier2025backward}: the user specifies a set size $T$, and BCP reports the miscoverage level supported by the data, using e-values, nonnegative statistics with expectation at most one, which remain valid when the level is selected post hoc \citep{vovk2025eprediction,gauthier2025evalues}. \citet{liu2026improving} tighten the reported level by transforming the nonconformity score. Fixing the set size is appropriate when a subsequent step can process only $T$ candidates, for example, when a clinician must evaluate a shortlist of diagnoses.

However, BCP requires exchangeable data, which fails under covariate shift, where test and calibration inputs follow different distributions, but the label distribution given the input is unchanged. Weighted conformal prediction handles covariate shift at a fixed level by weighting calibration scores with the test-to-calibration input density ratio \citep{tibshirani2019conformal}. To our knowledge, BCP has not been adapted to covariate shift.

We adapt BCP to covariate shift by weighting its e-values with the input density ratio, and show that its guarantees hold when the true ratio is known. For the two-point score of \citet{liu2026improving}, the resulting tight report equals the importance-weighted miss rate of the size-$T$ set on the calibration data, provided a non-degeneracy condition is met. In practice, the density ratio must be estimated, and the report depends on the weights only through the share of the total weight on the missed calibration points. Measures that average over all points, such as the mean squared error (MSE) of the weights, therefore do not determine its accuracy. In our experiments with a shift toward hard inputs, calibrating the classifier used to estimate the weights reduces their mean squared error and yet makes the tight report worse.

Because BCP relies only on the model's scores, we propose estimating the density ratio from the sorted scores rather than the original inputs. We prove that, in the limit, this gives the same report as weighting on the full input whenever the set misses equally often on test and calibration inputs with the same sorted scores. Otherwise, the two reports differ by the average change in that miss rate, which we call the reliability shift. More strongly, when the rank of the true label given the sorted scores has the same distribution under the shift, all of BCP's guarantees hold with this ratio. More broadly, the report's asymptotic bias has two sources: error in estimating the density ratio of the chosen representation, measured on the missed points, and bias introduced by that representation. Either sorted scores or an embedding can have lower bias, but the density ratio of sorted scores is typically easier to estimate.

We evaluate our approach on the Dry Bean dataset, where the shifts are known, and on CIFAR-10-C with two image models. Because image corruptions only approximate covariate shift, the CIFAR-10-C experiments provide an empirical test of our reports. On CIFAR-10-C, score weights bring the tight report closer to the true miscoverage than embedding weights for both models, across three set sizes and two overlap levels between clean and corrupted embeddings, and the remaining error tracks a binned estimate of the reliability shift. On CIFAR-10.1 \citep{recht2018cifar}, a natural shift of CIFAR-10, score weights again outperform embedding weights. The tight report still tends to understate the miscoverage, and average thresholded confidence (ATC; \citealp{garg2022leveraging}), which estimates a classifier's error from unlabeled data, is more accurate there as a point estimate of the miscoverage. Our method adds e-value guarantees for fixed-size sets under the true density ratio and an exact account of how estimated weights affect the report.

\section{Related Work}\label{sec:related}

\textbf{Conformal prediction under covariate shift.}
Weighted split conformal prediction weights the calibration scores by the covariate density ratio and constructs a set at a pre-specified level, achieving finite-sample coverage when the ratio is known \citep{tibshirani2019conformal}. \citet{barber2023beyond} bound the loss of coverage when exchangeability fails, and \citet{park2022pac} form PAC prediction sets using estimated importance weights and find that calibrating the weights, by binning them with uncertainty intervals, helps in their setting. Under exchangeability, conformal prediction has also been extended to fairness across groups in federated learning \citep{srinivasan2025fedcf}. These works fix the coverage level and allow set size to vary. In contrast, we fix the set size and report the level, and we compare with weighted split conformal prediction at matched mean set size.

\textbf{Backward conformal prediction and e-values.}
Conformal e-prediction replaces p-values with e-values, which lets levels be chosen after seeing the data \citep{vovk2025eprediction}. BCP uses this to fix the set size and report the level the data support, with a leave-one-out estimate of the expected level \citep{gauthier2025backward}, and ST-BCP \citep{liu2026improving} transforms the score to tighten the report. Both assume exchangeable data.

\textbf{Weighted e-values.}
Importance-weighted conformal e-values have been used in false discovery rate control for selection and outlier detection \citep{lee2024boosting}, as well as in selective prediction with risk control (SCoRE; \citealp{bai2026score}). Our weighted e-value uses the same change-of-measure principle, but its application to fixed-size sets and the resulting reports is novel. SCoRE selects predictions to trust at a fixed risk level, while BCP keeps every prediction, as a size-$T$ set, and reports its miscoverage. When weights are estimated, \citet{bai2026score} show that SCoRE remains asymptotically valid if the weights converge in $L^2$, and present a doubly robust approach. For BCP, we provide the exact finite-sample change in the report, which depends on the weights only through the share of the total weight on the missed calibration points. Weighted conformal p-values are also used for selective inference under covariate shift \citep{jin2023selection} and, via test martingales, for deployment monitoring \citep{prinster2025watch}. Fairness criteria can likewise be monitored at runtime with probabilistic guarantees, for which \citet{meghrazi2025interactive} provide an interactive interface.

\textbf{Estimating error under shift.}
Without assumptions on the shift, a classifier's accuracy on an unlabeled target distribution cannot be identified \citep{garg2022leveraging}; under covariate shift, importance weighting can estimate it (see, e.g., \citealp{cortes2010learning}). BCP's tight report serves as an importance-weighted estimate for the size-$T$ set, aligning it with other estimators of classifier error from unlabeled test data, such as ATC, the proportion of test points with confidence below a threshold determined from labeled source data \citep{garg2022leveraging}. We include ATC as a comparison. Risk estimation by reweighting on a low-dimensional function of the input is exact when the label is conditionally independent of the input given that function, and the loss depends on the input only through it \citep{stojanov2019low}. Mandoline reweights user-defined slices like predicted class and entropy \citep{chen2021mandoline}, while \citet{garg2022leveraging} include a confidence-reweighting baseline. Our reduction to sorted scores is the BCP analog of these approaches. Its condition holds exactly when the joint density ratio of the sorted scores and the rank of the true label depends only on the sorted scores; if not, we precisely characterize the asymptotic bias of the tight report.
\section{Weighted Backward Conformal Prediction}\label{sec:setting}

We build our analysis on the formalism of BCP \citep{gauthier2025backward}. Appendix~\ref{app:theory} contains all proofs and supporting results.

\begin{definition}[Setting]\label{def:setting}\StmtDefSetting\end{definition}

Definition~\ref{def:setting} describes covariate shift. The test distribution reweights the calibration distribution by a function $w$ of the input alone, so the label given the input has the same distribution under both. The score can be any fixed nonnegative function of the input and the label, for example the negative log-probability that a classifier assigns to the label.

\begin{definition}[Weighted e-value, level and set]\label{def:evalue}\StmtDefEvalue\end{definition}

The level $\tilde\alpha$ is the smallest level at which the set $\hat C$ has at most $T$ labels; when no level below one achieves this, $\tilde\alpha=1$ and the set can have more labels. With $w\equiv1$, the e-value, level and set are those of BCP. The test score enters the numerator without its weight to obtain a valid e-value (Appendix~\ref{app:design}).

\begin{theorem}[Weighted e-value validity]\label{thm:validity}\StmtThmValidity\end{theorem}

Theorem~\ref{thm:validity} extends the e-value used by \citet{balinsky2024enhancing} to covariate shift through the change of measure of weighted conformal prediction \citep{tibshirani2019conformal}. Because BCP chooses its level after seeing the test input, it needs the post-hoc form of this guarantee: the weighted version of \citet[Proposition~3.3]{gauthier2025evalues}.

\begin{theorem}[Post-hoc validity]\label{thm:posthoc}\StmtThmPosthoc\end{theorem}

The second inequality says that a miss, divided by the reported level, has expectation at most one, even though the level depends on the data. Under a further covariance condition, the expected level $\E_{\Pt}[\tilde\alpha]$ bounds the miscoverage itself (Theorem~\ref{thm:coverage}), a bound that \citet{gauthier2025backward} state to first order.

\section{Miscoverage Reports}\label{sec:tight}

BCP chooses the level $\tilde\alpha$ after seeing the test input, so the number to report is its expectation, $\E_{\Pt}[\tilde\alpha]$. It is unknown because we observe $\Pt$ only through unlabeled inputs. Following \citet{gauthier2025backward}, we estimate it by leave-one-out. Each calibration point is treated as the test point and yields one level; we then average these levels with weights so the weighted calibration sample approximates the test distribution. We use two such reports, one built on the score $S$ and one on the two-point score of ST-BCP \citep{liu2026improving}, and compare both with the miscoverage of the size-$T$ set.

\begin{definition}[Weighted e-value using the two-point score]\label{def:transform}\StmtDefTransform\end{definition}

The size-$T$ set $\topT(x)$ holds at most the $T$ labels with the smallest scores, and $m$ marks a miss. The two-point score $h$ is $0$ on $\topT(x)$ and $\tau(x)$ elsewhere, so it puts the whole e-value on the miss event. BCP's sets $\hat C$ and $\hat C_h$ equal $\topT(X_{n+1})$ whenever their levels are below one and a calibration score is positive. We compare every report with the miscoverage $\Pt(Y\notin\topT(X))$ of the size-$T$ set.

\begin{definition}[Reports]\label{def:reports}\StmtDefReports\end{definition}

With constant weights, the transformed report $\hat\alpha^{\rm tr}$ is the ST-BCP report. We report its values only in the appendix and use the other two here. For a positive function $u$ on $\mathcal X$, such as $w$ or an estimate $\hat\omega$, we write $\hat\alpha^{\rm id}(u)$ and $\hat b(u)$ for the reports with weights $\omega_j=u(X_j)$.

\textbf{The identity report.} 
With bounded true weights and scores, and $D^\star:=\E_P[w(X)S(X,Y)]>0$, $\hat\alpha^{\rm id}(w)-\E_{\Pt}[\tilde\alpha]=O_P(n^{-1/2})$ (Theorem~\ref{thm:loo}). Under the covariance condition of Theorem~\ref{thm:coverage}, its target $\E_{\Pt}[\tilde\alpha]$ bounds the miscoverage of $\hat C$, and hence that of $\topT(X_{n+1})$ when $\tilde\alpha<1$ almost surely.

\textbf{The tight report.}
ST-BCP proposes the corrected level, a Markov bound on the miscoverage of $\hat C_h$ that needs no condition on $\tilde\alpha_h$, and estimates it by leave-one-out. The authors recommend it when $T$ is small. Our tight report extends their estimate to covariate shift by weighting.

\begin{theorem}[Corrected level]\label{thm:corrected}\StmtThmCorrected\end{theorem}

With bounded true weights and scores, and $D^\star_h:=\E_P[w(X)h(X,Y)]>0$, the tight report converges to the expected corrected level, $\hat b(w)-\E_{\Pt}[\tilde\alpha_hE^{(w,h)}(Y_{n+1})]=O_P(n^{-1/2})$ (Theorem~\ref{thm:corrloo}). For the two-point score, both the level $\tilde\alpha_h$ and the corrected level have closed forms.

\begin{proposition}[Two-point closed form]\label{prop:twopoint}\StmtPropTwopoint\end{proposition}

The condition fails only on a miss whose $\tau(X_{n+1})$ is zero or below the weighted calibration average of $h$, and when it holds, the expected corrected level equals the miscoverage of the size-$T$ set. When the condition holds in all leave-one-out folds, which we call non-degeneracy, every miss counts in full and the tight report is the importance-weighted miss rate of the size-$T$ set on the calibration data,
\[
\hat b(\omega)=\frac{\sum_{j\le n}\omega_jm_j}{\sum_{j\le n}\omega_j}.
\]
This is the classical importance-weighted estimate of the size-$T$ set error. Weighted BCP adds the bound this estimate targets and e-value guarantees for its set under the true weights.

\section{Estimated Weights}\label{sec:estimated}

In practice, $w$ is estimated, typically with a domain classifier that separates unlabeled test inputs from held-out inputs of the calibration distribution \citep{tibshirani2019conformal}. The real challenge is ensuring that the estimated weights $\hat\omega_j$ represent the true weights $w_j$. However, under non-degeneracy, only the share of the total weight on the missed calibration points matters for the tight report. As a result, measures of weight quality that average over all points, such as the mean squared error of the weights or the calibration and accuracy of the domain classifier, can improve while the error of the report grows.
More formally, if the non-degeneracy condition holds for both the estimated and true weights, then
\[
\hat b(\hat\omega)-\hat b(w)=\frac1n\sum_{j=1}^n\Bigl(\frac{\hat\omega_j}{\bar{\hat\omega}}-\frac{w_j}{\bar w}\Bigr)m_j,
\]
where $\bar{\hat\omega}$ and $\bar w$ are the averages of the two weight vectors over the calibration points, so errors of the normalized weights on the covered calibration points do not enter the report error.
As $n$ grows, the report converges, and its asymptotic bias splits according to the representation the weights are fitted on.

\begin{proposition}[Asymptotic bias of the tight report]\label{prop:decomp}\StmtPropDecomp\end{proposition}

The assumption on $\tilde b_j$ is non-degeneracy with probability tending to one. The first term of the limit is the error in estimating the density ratio of $\phi(X)$, counted on the missed points. The second term is the bias of the representation. It is the negative of the average, over test inputs, of how much more often the set misses on test inputs than on calibration inputs with the same value of $\phi$. It is zero when $w$ depends on the input only through $\phi$, in particular when $\phi$ is the input itself. A coarser representation, a function of $\phi$, typically has a density ratio that is easier to estimate, but its bias can be larger or smaller than that of $\phi$ (Proposition~\ref{prop:compare}).

BCP uses a data point only through the model's scores and treats all labels alike. So, in weighted BCP, apart from the weights, its levels, miss events and reports depend on each data point only through its sorted scores and the rank of its label among them.

\begin{definition}[Sorted scores and rank]\label{def:GR}\StmtDefGr\end{definition}

$G(x)$ collects the scores of all $K$ candidate labels, so it is computed from the input alone, while the rank $R$ needs the label. The weights must be computed without the label, so $v$ conditions on $G$ alone; it is $v_\phi$ for $\phi=G$. We call weights fitted on $G(X)$ score weights. When we replace the full input $X$ with the score representation $G(X)$, a new bias is introduced, which is the negative of the reliability shift $\E_{\Pt}[\tilde\mu_G(G(X))-\mu_G(G(X))]$; such a shift measures how much more often the set misses on test inputs compared to calibration inputs with the same sorted scores. This change in the miss rate at fixed sorted scores is a concept shift in the score representation, which a covariate shift in the input can induce. The reliability shift is zero when the set misses equally often on test and calibration inputs with the same sorted scores. It is positive and biases the tight report downward when the set misses more often on test inputs, as on corrupted images on which the model stays confident. When the rank of the true label has the same conditional distribution given the sorted scores under $P$ and $\Pt$, BCP with weights $v\circ G$ retains the finite-sample guarantees of Sections~\ref{sec:setting} and~\ref{sec:tight}.

\begin{proposition}[Reduction]\label{prop:reduction}\StmtPropReduction\end{proposition}

The assumption holds when the density ratio depends on the input only through the sorted scores, $w=v\circ G$, and also when the conditional distribution of the rank, given the input, depends on the input only through the sorted scores (Proposition~\ref{prop:sufficient}).

The choice of representation trades the two terms of the limit. The input is unbiased, but its density ratio is hard to estimate in high dimensions. The model's penultimate embedding and its sorted scores both reduce the dimension. The sorted scores are a function of the embedding, since the logits are an affine map of it, yet either can have the smaller bias. What favors the sorted scores is estimation. They have $K$ scores per input, one per class, which is generally lower than the embedding dimension.

\textbf{Procedure.} We fit the domain classifier on the sorted scores of held-out inputs from the calibration distribution and unlabeled test inputs, evaluate the weights at the calibration points, and compute both reports. We use the tight report as the estimate of the miscoverage of the size-$T$ set and the identity report as the conservative number.

\section{Experiments}\label{sec:experiments}

We validate our theory on Dry Bean \citep{koklu2020multiclass}, where we can construct the shift, and compare estimated weights with true ones. Additionally, we use CIFAR-10-C \citep{hendrycks2019benchmarking}, where the benchmark shift comes from image corruptions with unknown true density ratios; we compare score weights with embedding weights. The nonconformity score we used is $S(x,y)=-\log\hat p(y\mid x)$ for a fixed classifier $\hat p$, with probabilities clipped at $10^{-12}$. A \emph{setting} is one shift, one model and one value of $T$. The gap of a report is the report minus the truth, the exact miscoverage of the size-$T$ set under the test distribution; and a report is \emph{valid} in a setting when its mean gap over the random draws is non-negative. On CIFAR-10-C, we also report the mean absolute gap in percentage points, the unit used for the errors of ST-BCP and ATC. Appendix~\ref{app:details} gives the remaining experimental details. 

For our simulated experimental setup on Dry Bean, calibration inputs are drawn uniformly and test inputs with probability proportional to $\exp(\beta t(x))$ for a tilt $t$ and a severity $\beta\in\{1,2\}$, so the true weight is $w(x)\propto\exp(\beta t(x))$. The tilts are the standardized entropy of the classifier's probabilities (\emph{hardness}), its negative (\emph{easy}), and the standardized feature ShapeFactor1, which is least correlated with the true-label score (\emph{benign}). We use $T\in\{1,2\}$. 

In our benchmark distribution shift setup on CIFAR-10 data \citep{krizhevsky2009learning}, we obtain calibration images from clean data and test images from CIFAR-10-C under Gaussian noise, defocus blur, fog, contrast, and JPEG compression at severities $1$, $3$, and $5$. We use the EfficientNet-B0 \citep{tan2019efficientnet} of \citet{gauthier2025backward} and a model trained for corruption robustness, the AugMix WideResNet-40-2 \citep{zagoruyko2016wide,hendrycks2020augmix} of RobustBench \citep{croce2021robustbench}, with $T\in\{1,2,3\}$ as in BCP. We call a model under one corruption at one severity a \emph{cell}. Each cell has $60$ draws of $n=1{,}000$ calibration images and of $2{,}000$ unlabeled corrupted versions of other images, and its truth is computed over all $10{,}000$ corrupted images. We also report the cells in which a classifier separates held-out clean and corrupted embeddings with AUC at least $0.95$, and call them \emph{walls}. A wall marks nearly separable embeddings, not necessarily a cell in which the model errs often. Additionally, as a natural shift, we take CIFAR-10.1 \citep{recht2018cifar}, a later test set collected to match CIFAR-10, with the same calibration pool and models as in the CIFAR-10-C experiments. Its results show the same pattern and are given in Appendix~\ref{app:extra}.

We use a gradient boosting domain classifier, fitted on the bean features (\emph{feature weights}), the model's embedding (\emph{embedding weights}), or the sorted scores (score weights). On Dry Bean, we compare the true, feature, and score weights; on CIFAR-10-C, we compare no weights ($w\equiv1$), embedding weights, and score weights. Without weights, the identity report is the estimator of BCP and the tight report is the corrected estimator of ST-BCP, so the results without weights are those of the two original methods. Weighted split conformal prediction (WSCP; \citealp{tibshirani2019conformal}) fixes the level and lets the set size vary, and we use the smallest level at which its mean set size on the unlabeled test sample is at most $T$. We also compare the tight report with ATC, using the negative entropy as the confidence (ATC-NE); at $T=1$ this is the method of \citet{garg2022leveraging} without temperature scaling.

\begin{table*}[t]
\centering\small
\caption{Dry Bean, $T=1$: exact miscoverage of the size-$1$ set under the tilted pool and gap (report $-$ miscoverage) of the
identity and tight reports, by the weights used, and of ATC-NE. ``Features'': gradient boosting on the $16$ features; ``+ Platt'': the
same classifier calibrated by Platt scaling on grouped out-of-fold predictions; ``sorted scores'': gradient boosting on
$G(x)$. Means over $80$ draws; the largest standard error in the table is \NdbTOnemaxse. $T=2$ in
Appendix~\ref{app:extra}.}\label{tab:drybean}
\resizebox{\linewidth}{!}{\begin{tabular}{lcccccc}
\toprule
weights & hardness $\beta=1$ & hardness $\beta=2$ & easy $\beta=1$ & easy $\beta=2$ & benign $\beta=1$ & benign $\beta=2$ \\
\midrule
true miscoverage & 0.274 & 0.406 & 0.023 & 0.010 & 0.072 & 0.036 \\
\midrule
\multicolumn{7}{l}{\emph{identity report}} \\
true weights & +0.169 & +0.343 & +0.014 & +0.014 & +0.012 & +0.023 \\
features & +0.013 & +0.106 & +0.028 & +0.027 & +0.009 & +0.025 \\
features + Platt & $-$0.064 & +0.046 & +0.059 & +0.042 & +0.025 & +0.034 \\
sorted scores & +0.267 & +0.340 & +0.014 & +0.017 & +0.016 & +0.021 \\
\multicolumn{7}{l}{\emph{tight report}} \\
true weights & $-$0.006 & $-$0.020 & 0.000 & 0.000 & $-$0.001 & 0.000 \\
features & $-$0.081 & $-$0.101 & +0.010 & +0.009 & $-$0.011 & $-$0.007 \\
features + Platt & $-$0.121 & $-$0.129 & +0.041 & +0.027 & +0.009 & +0.011 \\
sorted scores & +0.042 & $-$0.006 & $-$0.003 & $-$0.001 & +0.001 & 0.000 \\
\midrule
ATC-NE & +0.211 & +0.461 & $-$0.015 & $-$0.009 & $-$0.018 & $-$0.015 \\
\bottomrule
\end{tabular}
}
\end{table*}

\textbf{Tight report is accurate under true weights.}

With the true weights, the tight report is within \NdbTOneoraclecorrmaxabs{} of the miscoverage in every Dry Bean
setting at $T=1$, and the identity report stays above it (Table~\ref{tab:drybean}).

\begin{figure*}[t]
\centering
\includegraphics[width=0.9\linewidth]{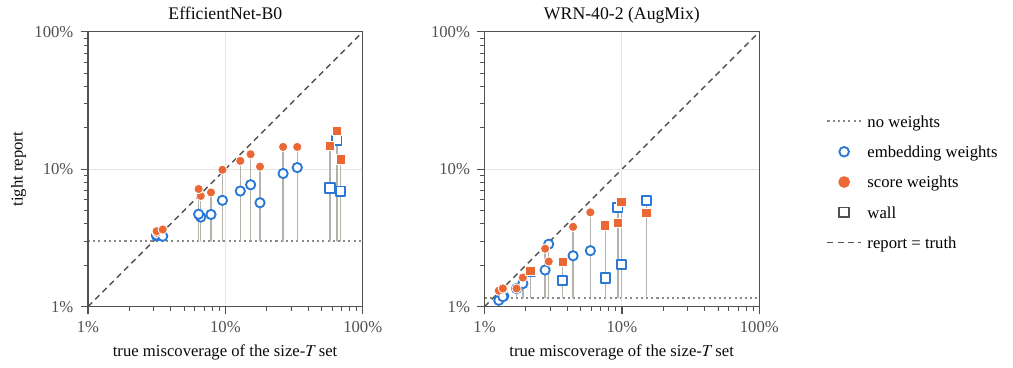}
\caption{CIFAR-10-C at $T=2$: the tight report against the true miscoverage of the size-$T$ set, one vertical segment per
setting, both axes logarithmic. Hollow blue markers: mean over the $60$ draws with embedding weights; filled red markers: with score
weights. Dotted line: no weights, the miss rate on the clean calibration images, the same in every setting. Squares:
walls, cells with nearly separable clean and corrupted embeddings (AUC at least $0.95$). Points below the dashed diagonal understate the miscoverage. $T=1$ and $T=3$ are in
Appendix~\ref{app:extra}.}\label{fig:cifar}
\end{figure*}

\textbf{What matters is the weight share of the missed points.} Estimated weights matter under the hardness shift, which moves the test distribution toward the beans the model finds hard; under the easy and benign shifts, both reports stay within \Ndbeasybenignmaxabs{} of the miscoverage with the true, feature, or score weights. Under the hardness shift at $T=1$, feature weights underweight the missed beans, and the tight report has gaps of \NdbhardnessOneTOneGBcorrected{} and \NdbhardnessTwoTOneGBcorrected{} at $\beta=1,2$. Calibrating the domain classifier lowers the MSE of the normalized weights for every method and setting, yet under the hardness shift the four calibration methods and logistic regression all worsen the tight report further (Appendix~\ref{app:calibration}). The MSE weighs all beans, while under non-degeneracy the tight report depends on the weights only through the share of the total weight on the missed beans. Score weights bring the gaps to \NdbhardnessOneTOneSCOREScorrected{} and \NdbhardnessTwoTOneSCOREScorrected{}. The hardness and easy tilts are functions of the sorted scores $G$, so the assumption of Proposition~\ref{prop:reduction} holds and these two shifts validate the theory rather than test the method. The benign tilt is not a function of $G$, and score weights correlate less with the true weights than feature weights do (Appendix~\ref{app:calibration}), yet the tight report with score weights stays within \Ndbbenignscorecorrmaxabs{} of the miscoverage, which this shift moves from \NdbTOneunweightedtruth{} to \NdbbenignTwoTOnetruth{} at $T=1$, $\beta=2$.

\begin{table*}[t]
\centering\small
\caption{CIFAR-10-C by model, $T$ and overlap (walls: held-out embedding AUC at least $0.95$); each row averages over the
settings of its group. Tight report with no, embedding and score weights, and
ATC-NE: mean absolute gap $|\text{report}-\text{truth}|$ in percentage points over the settings and their $60$ draws;
without weights the tight report is the corrected estimator of ST-BCP. Identity report: settings in which it is valid,
with score weights. Appendix~\ref{app:extra} gives the signed gaps by group, and per-setting gaps of the weighted reports
and ATC-NE with standard errors.}\label{tab:main}
\begin{tabular}{lllcccccc}
\toprule
 & & & true & \multicolumn{3}{c}{tight report} & & identity report \\
\cmidrule(lr){5-7}
model & $T$ & & miscoverage & none & emb. & score & ATC-NE & valid \\
\midrule
EfficientNet-B0 & 1 & no wall & 23.6\% & 14.69 & 8.15 & 4.74 & 3.12 & 9/12 \\
 &  & walls & 77.0\% & 68.13 & 52.34 & 42.98 & 29.59 & 0/3 \\
\addlinespace
 & 2 & no wall & 12.2\% & 9.20 & 6.65 & 4.78 & 3.49 & 10/12 \\
 &  & walls & 63.8\% & 60.80 & 53.65 & 48.67 & 41.87 & 0/3 \\
\addlinespace
 & 3 & no wall & 7.0\% & 5.67 & 4.28 & 3.48 & 2.57 & 11/12 \\
 &  & walls & 50.4\% & 49.08 & 43.90 & 42.39 & 38.66 & 0/3 \\
\midrule
WRN-40-2 (AugMix) & 1 & no wall & 7.9\% & 3.13 & 1.95 & 1.61 & 0.83 & 0/9 \\
 &  & walls & 17.4\% & 12.53 & 9.08 & 6.06 & 3.98 & 0/6 \\
\addlinespace
 & 2 & no wall & 2.6\% & 1.51 & 1.24 & 0.93 & 0.69 & 9/9 \\
 &  & walls & 8.0\% & 6.80 & 5.96 & 4.36 & 3.44 & 4/6 \\
\addlinespace
 & 3 & no wall & 1.2\% & 0.74 & 0.76 & 0.63 & 0.52 & 9/9 \\
 &  & walls & 4.4\% & 3.90 & 3.85 & 2.84 & 2.40 & 5/6 \\
\bottomrule
\end{tabular}

\end{table*}

\textbf{Score weights better match CIFAR-10-C truth.} Score weights bring the tight report closer to the miscoverage than embedding weights or no weights across all \Ncifmaengroups{} groups of the model, $T$, and overlap (Table~\ref{tab:main}). Without weights, the tight report is the miss rate on the clean calibration images, the same in all $15$ cells of a model at each $T$ (the dotted lines in Figure~\ref{fig:cifar}), so the estimators of BCP and ST-BCP do not follow the shift. Setting by setting, the mean tight report with score weights is closer to the miscoverage than with embedding weights in \NciftightcloserEMB{} of the \Ncifnsettings{} settings, and the paired bootstrap $95\%$ interval of the difference over the $60$ draws excludes zero in \NciftightbootSCORE{} settings in favor of score weights and in \NciftightbootEMB{} in favor of embedding weights. Score weights are also closer in all \Ncifmaengroups{} groups with three other domain classifiers and with looser clips of the weights (Appendix~\ref{app:extra}). Clean and corrupted images overlap more on the sorted scores; a classifier trained as in the wall criterion separates them with a lower held-out AUC on the sorted scores in all \Ncifaucncells{} cells.

\textbf{Residual error tracks a binned reliability shift.} The tight report with score weights still understates the miscoverage in \Nciftightunder{} of the \Ncifnsettings{} settings. With score weights, Proposition~\ref{prop:decomp} gives its asymptotic bias as the error of the weights on the missed points minus the reliability shift. A one-dimensional estimate of this shift, computed with the test labels, tracks the error. We bin the clean images by their top-$T$ softmax mass and predict the corrupted miss rate from the clean miss rate in each bin. With both errors divided by miscoverage, the prediction error and the error of the tight report have correlation \Nreloldrelcorr{} (EfficientNet-B0) and \Nrelnewrelcorr{} (WideResNet) across each model's settings (Appendix~\ref{app:extra}). The WideResNet has the smaller reliability shift. Summed over the settings without walls, its absolute binned shift is \Nrelnewnonwallrelsize{} of its miscoverage,  compared with \Nreloldnonwallrelsize{} for EfficientNet-B0, and its tight report with score weights is less biased, with gaps summing to \NcifnewnonwallrelbiasSCORE\% of the miscoverage, compared with \NcifoldnonwallrelbiasSCORE\%.

\textbf{The identity report is conservative except at $T=1$ on CIFAR-10-C and near the walls.} With score weights, the identity report is valid in all $12$ Dry Bean settings and in \NcifidvalidTtwothreenonwall{} of the \NcifidncellsTtwothreenonwall{} CIFAR-10-C settings without a wall at $T=2,3$, but in only \NciftotTOnenonwallSCOREidentityvalid{} of the \NciftotTOnenonwallncells{} at $T=1$. Two of the three failures at $T=2,3$ are defocus blur at severity $5$ for EfficientNet-B0, whose embedding AUC (\Nciffailauc) is just below the wall threshold. At the walls, it is valid in \NcifidvalidTtwothreewall{} of the \NcifidncellsTtwothreewall{} settings at $T=2,3$ and in none at $T=1$. The failure at $T=1$ occurs even without shift. With clean calibration and test images and no weights, the identity report lies \NnoshnewToneidbelow{} points below the top-1 miscoverage of the WideResNet, whose scores are never clipped, while the tight report is within \Nnoshtightmaxabs{} points (Appendix~\ref{app:tone}). At $T=1$, a covered true label is the model's top prediction, whose score is near zero, so covered points provide little slack; for EfficientNet-B0, the sign of the gap depends on our score clip.

\textbf{Comparison with ATC and WSCP.}  ATC-NE estimates the miscoverage more accurately than the tight report with score weights in all \Ncifmaeatcbeatsscore{} groups (Table~\ref{tab:main}), but it has no guarantee under covariate shift in general. It counts every test point below a confidence threshold as a miss, which is exact only if the threshold transfers to the test distribution \citep{garg2022leveraging}. On Dry Bean, where the covariate shift is exact, it overstates the miscoverage under the hardness shift, with gaps of \NdbhardnessOneTOneATC{} and \NdbhardnessTwoTOneATC{} at $T=1$, and is further from the true miscoverage than the tight report with score weights in all \Ndbatcfarther{} settings (Table~\ref{tab:drybean}, Appendix~\ref{app:extra}). Summed over the walls of each model and $T$, the tight report with score or embedding weights and ATC-NE all understate the summed miscoverage by \Ncifwallrelmin\% to \Ncifwallrelmax\%. WSCP keeps its miscoverage level with the true weights on Dry Bean (largest excess \Ndbwscporaclemaxexcess) but not with estimated weights. Under the hardness shift at $T=1$, its miscoverage exceeds $\alpha$ by up to \NdbwscpGBhardTwoexcess{} with feature weights, and on CIFAR-10-C with score weights it keeps its level in only \NcifoldTTwononwallWSCPvalid{} of the \NcifoldTTwononwallncells{} EfficientNet-B0 and \NcifnewTTwononwallWSCPvalid{} of the \NcifnewTTwononwallncells{} WideResNet settings without a wall at $T=2$. Fixing its level in advance does not remove this excess (Appendix~\ref{app:extra}). At $T=2$, \NwscpTTwooverlo--\NwscpTTwooverhi\% of the test images receive more than $T$ labels, and it misses less often than the size-$T$ set (median ratio \NwscpoldTTworatiomedian{} and \NwscpnewTTworatiomedian{} for the two models). With the true weights on Dry Bean at $\beta=2$, however, it misses more often than the size-$T$ set in all $6$ settings.

\section{Conclusion \& Discussion}\label{sec:discussion}

\textbf{Conclusion.} We extended BCP to covariate shift. Weighted BCP fixes the set size and reports its miscoverage, with BCP guarantees under the true density ratio. Under non-degeneracy, its tight report depends on estimated weights only through their share on the missed calibration points, so a domain classifier with lower weight MSE can give a worse report. On CIFAR-10-C, score weights brought the tight report closer to the miscoverage than embedding weights in every group of the model, $T$ and overlap. ATC was more accurate there as a point estimate, but it has no guarantee under covariate shift in general.

\textbf{Limitations.} The finite-sample guarantees need the true density ratio, or the density ratio of the sorted scores under the condition of Proposition~\ref{prop:reduction}. With estimated weights, we describe the report error but do not bound it, and our consistency results are asymptotic. The tight report equals the weighted miss rate only under non-degeneracy, which fails when the miscoverage is high. Even with calibration data from the test distribution itself, the tight report still understates the miscoverage by up to \Nndlimitmaxshort{} points for EfficientNet-B0 at $T=1$ (Appendix~\ref{app:extra}). Neither the condition of Proposition~\ref{prop:reduction} nor the reliability shift can be checked from the sorted scores and the calibration labels alone. The margin of the identity report above the miscoverage is an empirical finding, and at $T=1$ on CIFAR-10-C it fails in most settings. At the walls, where embeddings of calibration and test data barely overlap, the tight report with score or embedding weights, together with ATC-NE, understates the summed miscoverage by \Ncifwallrelmin\% to \Ncifwallrelmax\%, and our criterion for walls is a heuristic. On CIFAR-10-C and CIFAR-10.1, the density ratio assumed by the theory is at best an approximation. We clip probabilities at $10^{-12}$ to bound the scores, as our consistency results assume; for EfficientNet-B0, this lowers the identity report at $T=1$ (Appendix~\ref{app:tone}).

\textbf{Open questions.} The reliability shift is the bias that score weights leave, and estimating it needs either the input-space ratio or labels from the test distribution. Our binned estimate used the test labels and tracked the remaining error closely; whether a small labeled test sample suffices to correct the report is open. The theory does not decide between the sorted scores and an embedding, and a representation chosen to balance the two terms of Proposition~\ref{prop:decomp} might do better than either. A check of when estimated weights can be trusted is also missing. This work does not cover adaptive set sizes, which BCP explores.

\bibliographystyle{plainnat}
\bibliography{refs}

\clearpage
\onecolumn
\appendix
\begin{center}{\LARGE\bfseries Appendix}\end{center}
\counterwithin{theorem}{section}\counterwithin{proposition}{section}\counterwithin{lemma}{section}\counterwithin{remark}{section}
\section{Proofs and additional results}\label{app:theory}

\subsection{Setting}
\paragraph{Definition~\ref{def:setting} (Setting), restated.} \StmtDefSetting

\noindent Since $\Pt$ is a probability measure, $\E_P[w(X)]=1$.
\paragraph{Definition~\ref{def:evalue} (Weighted e-value, level and set), restated.} \StmtDefEvalue

\begin{lemma}[Closed form of $F$]\label{lem:F}\StmtLemF\end{lemma}
\begin{proof}For $t>0$, $\#\{y:e_y<t\}\le T$ if and only if $e_{(T+1)}\ge t$. So the set in the definition of $F$ is
$\{\alpha\in(0,1):\alpha\ge1/e_{(T+1)}\}$, whose infimum is $1/e_{(T+1)}$ if $e_{(T+1)}>1$ and which is empty
otherwise. The map $e\mapsto e_{(T+1)}$ is $1$-Lipschitz in $\max_y|\cdot|$, $x\mapsto\min\{1,1/x\}$ is $1$-Lipschitz on
$[0,\infty)$, and $F$ takes values in $(0,1]$.\end{proof}

\subsection{Validity with the true weights}
\paragraph{Theorem~\ref{thm:validity} (Weighted e-value validity), restated.} \StmtThmValidity
\begin{proof}Let $\Sigma:=\sum_{j=1}^{n+1}w_jS_j$. At the true label $E^{(w)}(Y_{n+1})=W\,S_{n+1}/\Sigma$ on $\{\Sigma>0\}$, and on
$\{\Sigma=0\}$ it is $0$ ($S_{n+1}=0$ there). Only the test point is drawn from $\Pt$; replacing its law by $P$ and
inserting $w_{n+1}$ gives, with $(X_j,Y_j)_{j\le n+1}$ i.i.d.\ from $P$,
$\E_{\Pt}[E^{(w)}(Y_{n+1})]=\E_P[W\,w_{n+1}S_{n+1}\Ind\{\Sigma>0\}/\Sigma]$. The integrand is at most $W$ and
$\E_P[W]=n+1$. The joint law is exchangeable and $W$, $\Sigma$ are symmetric, so $n+1$ can be replaced by any index
$k$; averaging over $k$, the expectation equals $\frac1{n+1}\E_P[W\Ind\{\Sigma>0\}]=1-\frac1{n+1}\E_P[W\Ind\{\Sigma=0\}]$.
Since $w>0$, $\{\Sigma=0\}=\{S_j=0\ \forall j\le n+1\}$, so by independence
$\frac1{n+1}\E_P[W\Ind\{\Sigma=0\}]=\E_P[w\Ind\{S=0\}]\,P(S=0)^n=\Pt(S=0)\,P(S=0)^n$. Hence
$\E_{\Pt}[E^{(w)}(Y_{n+1})]=1-\Pt(S=0)\,P(S=0)^n\le1$, with equality when $\Pt(S=0)=0$.\end{proof}

\paragraph{Theorem~\ref{thm:posthoc} (Post-hoc validity), restated.} \StmtThmPosthoc
\begin{proof}For $e\ge0$ and $a>0$, $\Ind\{e\ge 1/a\}/a\le e$. Take $e=E^{(w)}(Y_{n+1})$, $a=\hat\alpha$, and use
Theorem~\ref{thm:validity}. The second claim is the first with $\hat\alpha=\tilde\alpha\in(0,1]$ (Lemma~\ref{lem:F}),
because $Y_{n+1}\notin\hat C$ iff $E^{(w)}(Y_{n+1})\ge1/\tilde\alpha$.\end{proof}

\begin{theorem}[Miscoverage bound]\label{thm:coverage}\StmtThmCoverage\end{theorem}
\begin{proof}$\E_{\Pt}[I]=\E_{\Pt}[\tilde\alpha]\,\E_{\Pt}[I/\tilde\alpha]+\mathrm{Cov}_{\Pt}(\tilde\alpha,I/\tilde\alpha)
\le\E_{\Pt}[\tilde\alpha]$ by Theorem~\ref{thm:posthoc} and the assumption.\end{proof}

\begin{remark}\label{rem:cov}\StmtRemCov\end{remark}

\begin{proposition}[Unconditional bound]\label{prop:sqrt}\StmtPropSqrt\end{proposition}
\begin{proof}With $I$ as above, $I=I^2$ and Cauchy--Schwarz give
$\E_{\Pt}[I]\le(\E_{\Pt}[I/\tilde\alpha])^{1/2}(\E_{\Pt}[\tilde\alpha])^{1/2}$, and $\E_{\Pt}[I/\tilde\alpha]\le1$
by Theorem~\ref{thm:posthoc}.\end{proof}

\subsection{The two-point score and the corrected level}
\paragraph{Definition~\ref{def:transform} (Weighted e-value using the two-point score), restated.} \StmtDefTransform

\noindent The score $h$ is the transformation $I_w$ of \citep[eq.~(21)]{liu2026improving} with threshold $\tau$.
\begin{theorem}[Validity of the transformed e-value]\label{thm:transform}\StmtThmTransform\end{theorem}
\begin{proof}Theorem~\ref{thm:validity} applied to the score $h\ge0$.\end{proof}

\begin{lemma}[Markov bound]\label{lem:markov}\StmtLemMarkov\end{lemma}
\begin{proof}$\Ind\{e\ge1/\hat\alpha\}\le\hat\alpha e$ pointwise. (This is the ``corrected coverage level'' of \citep[Appendix~E]{liu2026improving}.)\end{proof}

\paragraph{Theorem~\ref{thm:corrected} (Corrected level), restated.} \StmtThmCorrected
\begin{proof}Lemma~\ref{lem:markov} with $e=E^{(w,h)}(Y_{n+1})$ and $\hat\alpha=\tilde\alpha_h$.\end{proof}

\paragraph{Proposition~\ref{prop:twopoint} (Two-point closed form), restated.} \StmtPropTwopoint
\begin{proof}$E^{(w,h)}(y)=0$ for $y\in\topT(X_{n+1})$, at most $T$ labels, and $E^{(w,h)}(y)=c$ otherwise. So the $(T+1)$-th
smallest entry is $c$, and Lemma~\ref{lem:F} gives $\tilde\alpha_h=\min\{1,1/c\}$. At the true label
$E^{(w,h)}(Y_{n+1})=c\,m_{n+1}$, so $\tilde\alpha_hE^{(w,h)}(Y_{n+1})=\min\{1,1/c\}c\,m_{n+1}=\min\{1,c\}m_{n+1}$. Under
the extra assumption this equals $m_{n+1}$ $\Pt$-a.s.; the last equality is the change of measure.
The unweighted closed form is \citep[Corollary~3.4]{liu2026improving}.\end{proof}

\subsection{Two reports of the miscoverage}
\paragraph{Definition~\ref{def:reports} (Reports), restated.} \StmtDefReports

\noindent The proof of Proposition~\ref{prop:twopoint} uses only the form of $E^{(w,h)}$ and Lemma~\ref{lem:F}, so it
applies to fold $j$ with any weights $\omega$: $\tilde b_j(\omega)=\min\{1,c_j(\omega)\}m_j$ with
$c_j(\omega):=(\sum_i\omega_i)\tau(X_j)/(\sum_{i\ne j}\omega_ih_i+\omega_j\tau(X_j))$. For $n\ge2$,
$\tilde b_j(\omega)=m_j$ iff $m_j=0$, or $\tau(X_j)>0$ and $\tau(X_j)\ge\sum_{i\ne j}\omega_ih_i/\sum_{i\ne j}\omega_i$.
We call $\tilde b_j(\omega)=m_j$ for every $j\le n$ non-degeneracy of the weights $\omega$. The unweighted tight
report $\frac1n\sum_j\tilde b_j$ is the estimator of \citep[Appendix~E]{liu2026improving}.
\begin{remark}[Non-degeneracy]\label{rem:nondeg}\StmtRemNondeg\end{remark}
\begin{proof}If $m_j=1$ then $h_j=\tau(X_j)$, so the denominator of $c_j(\omega)$ is $\sum_i\omega_ih_i$ and
$c_j(\omega)=\tau(X_j)/\bar h(\omega)$; if $\bar h(\omega)=0$ then $\tau(X_j)=0$ and both sides are $0/0=0$. By the note
after Definition~\ref{def:reports}, $\tilde b_j(\omega)=\min\{1,c_j(\omega)\}m_j$, which gives the display. Since
$\omega_j>0$, equality holds iff $c_j(\omega)\ge1$ for every $j$ with $m_j=1$, that is, iff $\tilde b_j(\omega)=m_j$ for
every $j\le n$. Let $m_j=1$. If $\tau(X_j)>0$ then $\bar h(\omega)>0$, and $c_j(\omega)\ge1$ iff
$\tau(X_j)\ge\bar h(\omega)$. If $\tau(X_j)=0$ then $c_j(\omega)=0$. This gives the second equivalence, and the second
claim follows from it. For the third claim, $\bar h(\omega)$ is the ratio of the averages of
$\hat\omega(X_i)h_i\in[0,\omega_{\max}S_{\max}]$ (since $h\le\tau\le S_{\max}$) and of
$\hat\omega(X_i)\in(0,\omega_{\max}]$ over $i\le n$. By Chebyshev
these converge in probability to $\E_P[\hat\omega h]$ and $\int\hat\omega\,\mathrm dP_X>0$, so
$\bar h(\omega)\to\E_P[\tilde\omega h]<\tau_{\min}$ in probability and $\mathbb P(\bar h(\omega)\le\tau_{\min})\to1$. By the
second claim, $\tilde b_j(\omega)=m_j$ for every $j\le n$ on this event.\end{proof}

\noindent Non-degeneracy thus compares the threshold of each missed point with one weighted average. With
$r(\omega):=\sum_j\omega_jm_j/\sum_j\omega_j$ the weighted miss rate and $\bar\tau(\omega)$ the $\omega$-weighted mean of
$\tau$ over the missed points, $\bar h(\omega)=r(\omega)\bar\tau(\omega)$, so non-degeneracy fails exactly when the miss
rate exceeds the ratio of the smallest to the mean threshold among the missed points. For $S(x,y)=-\log\hat p(y\mid x)$
with a probability vector $\hat p(\cdot\mid x)$ clipped from below at a constant $\varepsilon\in(0,1/(T+1)]$, the
$(T+1)$-th largest probability at $x$ is at most $1/(T+1)$, so the assumptions $\tau\ge\tau_{\min}$ and $S\le S_{\max}$ of
Remark~\ref{rem:nondeg} hold with $\tau_{\min}=\log(T+1)$ and $S_{\max}=-\log\varepsilon$.
\begin{theorem}[Consistency of the identity report]\label{thm:loo}\StmtThmLoo\end{theorem}
\begin{proof}Let $B_n:=\frac1n\sum_{j\le n}w_j$ and $\bar S:=\frac1n\sum_{i\le n}w_iS_i$. Then $\E B_n=1$,
$\mathrm{Var}(B_n)\le w_{\max}/n$, $\E\bar S=D^\star$ and $\mathrm{Var}(\bar S)\le w_{\max}S_{\max}^2/n$. Let
$A_n:=\{B_n\ge1/2,\ \bar S\ge D^\star/2\}$; by Chebyshev $\mathbb P(A_n^c)=O(1/n)$. Let $n_0$ be such that
$w_{\max}S_{\max}/n\le D^\star/4$ for $n\ge n_0$. We define the surrogate $\bar e(x)\in[0,\infty)^{\mathcal Y}$ by
$\bar e(x)_y:=S(x,y)/D^\star$ and $\bar a(x):=F(\bar e(x))$.

\emph{Calibration folds.} $E^{(w),j}(y)=B_nS(X_j,y)/D_j(y)$ with
$D_j(y):=\bar S+\frac{w_j}n(S(X_j,y)-S_j)$ and $|D_j(y)-\bar S|\le w_{\max}S_{\max}/n$. On $A_n$ and for $n\ge n_0$,
$D_j(y)\ge D^\star/4$ and
\[
\bigl|E^{(w),j}(y)-\bar e(X_j)_y\bigr|=S(X_j,y)\frac{|B_nD^\star-D_j(y)|}{D_j(y)D^\star}
\le\frac{4S_{\max}}{D^{\star2}}\Bigl(D^\star|B_n-1|+|\bar S-D^\star|+\frac{w_{\max}S_{\max}}n\Bigr)=:\Delta_n .
\]
By Lemma~\ref{lem:F}, on $A_n$ and for $n\ge n_0$,
$\bigl|\frac1n\sum_jw_j\tilde\alpha_j-\frac1n\sum_jw_j\bar a(X_j)\bigr|\le B_n\Delta_n$, and $\Delta_n=O_P(n^{-1/2})$.
Since $\mathbb P(A_n^c)=O(1/n)$, a bound that holds on $A_n$ for $n\ge n_0$ yields an $O_P$ bound.
The summands $w_j\bar a(X_j)$ are i.i.d.\ in $[0,w_{\max}]$, so
$\frac1n\sum_jw_j\bar a(X_j)-\E_P[w\bar a(X)]=O_P(n^{-1/2})$, and $\E_P[w\bar a(X)]=\E_{\Pt}[\bar a(X_{n+1})]$.

\emph{Target.} $E^{(w)}(y)=\frac Wn S(X_{n+1},y)/(\bar S+\frac{w_{n+1}}nS(X_{n+1},y))$ and $W/n=B_n+w_{n+1}/n$. On
$A_n$ and for $n\ge n_0$ the denominator is at least $D^\star/2$ and
$|E^{(w)}(y)-\bar e(X_{n+1})_y|\le\frac{2S_{\max}}{D^{\star2}}\bigl(D^\star|B_n-1|+|\bar S-D^\star|+\frac{w_{\max}(D^\star+S_{\max})}n\bigr)$.
$A_n$ depends on the calibration data only. By Lemma~\ref{lem:F},
\[
\bigl|\E_{\Pt}[\tilde\alpha]-\E_{\Pt}[\bar a(X_{n+1})]\bigr|
\le\mathbb P(A_n^c)+\frac{2S_{\max}}{D^{\star2}}\Bigl(D^\star\E|B_n-1|+\E|\bar S-D^\star|+\frac{w_{\max}(D^\star+S_{\max})}n\Bigr)
=O(n^{-1/2}),
\]
using $\E|Z-\E Z|\le\mathrm{Var}(Z)^{1/2}$.

\emph{Conclusion.} With $L:=\E_{\Pt}[\tilde\alpha]\in(0,1]$,
$\hat\alpha^{\rm id}(w)-L=\bigl(\frac1n\sum_jw_j\tilde\alpha_j-L B_n\bigr)/B_n$, the numerator is
$O_P(n^{-1/2})$ by the three estimates and $B_n-1=O_P(n^{-1/2})$, and $B_n\to1$ in probability.\end{proof}

\noindent With $w\equiv1$, the same rate was shown under stronger assumptions \citep[Theorem~3.1]{gauthier2025backward}, namely
scores bounded away from zero, $n>S_{\max}/S_{\min}$ and their conditions (P1)--(P2).
\begin{theorem}[Consistency of the tight report]\label{thm:corrloo}\StmtThmCorrloo\end{theorem}
\begin{proof}Let $B_n$ be as above and $\bar H:=\frac1n\sum_iw_ih_i$, with $\E\bar H=D^\star_h$ and
$\mathrm{Var}(\bar H)\le w_{\max}S_{\max}^2/n$. Let $A_n:=\{B_n\ge1/2,\ \bar H\ge D^\star_h/2\}$, so
$\mathbb P(A_n^c)=O(1/n)$. By the note after Definition~\ref{def:reports}, $\tilde b_j=\min\{1,c_j\}m_j$ with
$c_j=B_n\tau_j/\hat H_j$, $\hat H_j:=\bar H+\frac{w_j}n(\tau_j-h_j)\in[\bar H,\bar H+w_{\max}S_{\max}/n]$. We define
$\bar c(x):=\tau(x)/D^\star_h$ and $\bar b(x,y):=\min\{1,\bar c(x)\}m(x,y)$. On $A_n$,
\[
|c_j-\bar c(X_j)|=\tau_j\frac{|B_nD^\star_h-\hat H_j|}{\hat H_jD^\star_h}
\le\frac{2S_{\max}}{D^{\star2}_h}\Bigl(D^\star_h|B_n-1|+|\bar H-D^\star_h|+\frac{w_{\max}S_{\max}}n\Bigr),
\]
and $|\tilde b_j-\bar b(X_j,Y_j)|\le|c_j-\bar c(X_j)|$ because $\min\{1,\cdot\}$ is $1$-Lipschitz and $m_j$ is shared.
Hence, on $A_n$, $\bigl|\frac1n\sum_jw_j\tilde b_j-\frac1n\sum_jw_j\bar b(X_j,Y_j)\bigr|\le B_n\Delta^h_n$, where
$\Delta^h_n$ is the right-hand side of the display, and $\Delta^h_n=O_P(n^{-1/2})$. The summands
$w_j\bar b(X_j,Y_j)$ are i.i.d.\ in $[0,w_{\max}]$, so
$\frac1n\sum_jw_j\bar b(X_j,Y_j)-\E_P[w\bar b(X,Y)]=O_P(n^{-1/2})$, and
$\E_P[w\bar b(X,Y)]=\E_{\Pt}[\bar b(X_{n+1},Y_{n+1})]$.
\emph{Target.} By Proposition~\ref{prop:twopoint}, $\beta^\star=\E_{\Pt}[\min\{1,c\}m_{n+1}]$ with
$c=\frac Wn\tau_{n+1}/(\bar H+\frac{w_{n+1}}n\tau_{n+1})$. On $A_n$ the denominator is at least $D^\star_h/2$ and
$|c-\bar c(X_{n+1})|\le\frac{2S_{\max}}{D^{\star2}_h}\bigl(D^\star_h|B_n-1|+|\bar H-D^\star_h|
+\frac{w_{\max}(D^\star_h+S_{\max})}n\bigr)$. $A_n$ depends on the calibration data only, and
$|\min\{1,a\}-\min\{1,b\}|\le\min\{1,|a-b|\}$, so
$\bigl|\beta^\star-\E_{\Pt}[\bar b(X_{n+1},Y_{n+1})]\bigr|\le\mathbb P(A_n^c)+\frac{2S_{\max}}{D^{\star2}_h}
\bigl(D^\star_h\E|B_n-1|+\E|\bar H-D^\star_h|+\frac{w_{\max}(D^\star_h+S_{\max})}n\bigr)=O(n^{-1/2})$.
\emph{Conclusion.} $\hat b(w)-\beta^\star=\bigl(\frac1n\sum_jw_j\tilde b_j-\beta^\star B_n\bigr)/B_n$ with
$\beta^\star\in[0,1]$; the numerator is $O_P(n^{-1/2})$ by the three estimates and $B_n-1=O_P(n^{-1/2})$, and
$B_n\to1$ in probability.\end{proof}

\begin{remark}$\mathrm{Var}(B_n)=\chi^2(\Pt_X\|P_X)/n$, $\mathrm{Var}(\bar S)\le S_{\max}^2\E_P[w^2]/n$ and
$\mathrm{Var}(\bar H)\le S_{\max}^2\E_P[w^2]/n$; the proofs of
Theorems~\ref{thm:loo} and~\ref{thm:corrloo} bound both through $w_{\max}$, and their constants also depend on
$S_{\max}$ and $D^\star$ (resp.\ $D^\star_h$); we do not claim that the effective sample size alone governs the error.\end{remark}

\subsection{Estimated weights}
\begin{proposition}[Finite-sample identity]\label{prop:fs}\StmtPropFs\end{proposition}
\begin{proof}Under the assumption $\hat b(\hat\omega)=\sum_j\hat\omega_jm_j/\sum_j\hat\omega_j=\frac1n\sum_j(\hat\omega_j/\bar{\hat\omega})m_j$,
and likewise for $\omega$; subtract.\end{proof}

\begin{proposition}[Limit under estimated weights]\label{prop:west}\StmtPropWest\end{proposition}
\begin{proof}Let $\Omega_n$ be the event in the assumption and $Q_n:=\frac1n\sum_j\hat\omega(X_j)m_j/\frac1n\sum_j\hat\omega(X_j)$.
On $\Omega_n$, $\hat b(\omega)=Q_n$. Numerator and denominator are i.i.d.\ averages bounded by $\omega_{\max}$; by Chebyshev they
converge in probability to $\E_P[\hat\omega m]$ and $\int\hat\omega\,\mathrm dP_X>0$, so
$Q_n\to\E_P[\tilde\omega m]$ in probability. For $\varepsilon>0$,
$\mathbb P(|\hat b-\E_P[\tilde\omega m]|>\varepsilon)\le\mathbb P(\Omega_n^c)+\mathbb P(|Q_n-\E_P[\tilde\omega m]|>\varepsilon)\to0$.
The second claim subtracts the constant $\Pt(Y\notin\topT(X))=\E_P[wm]$.\end{proof}

\noindent Under non-degeneracy the tight report is the self-normalized importance-weighted estimate of the top-$T$
error, so Proposition~\ref{prop:west} is the classical limit of that estimator under misspecified weights.
\paragraph{Proposition~\ref{prop:decomp} (Asymptotic bias of the tight report), restated.} \StmtPropDecomp
\begin{proof}The versions exist by the factorization lemma. Since $w>0$, $P_X$ and $\Pt_X$ have the same null sets, so no
expectation below depends on the choice of versions. All of them are finite, since $0\le m\le1$, $\tilde\omega$ is
bounded and $\E_P[v_\phi(\phi(X))]=\E_P[w(X)]=1$. By Proposition~\ref{prop:west} the limit is
$\E_P[(\tilde\omega-w)m]=\E_P[(\tilde\omega-v_\phi\circ\phi)m]+\E_P[(v_\phi\circ\phi-w)m]$, and the last term equals
$-\E_{\Pt}[\tilde\mu_\phi(\phi(X))-\mu_\phi(\phi(X))]$ by the first identity of Proposition~\ref{prop:compare}.\end{proof}

\noindent With $\phi$ the identity, $v_\phi=w$ and, under covariate shift, $\tilde\mu_\phi=\mu_\phi$, so the second term
vanishes and Proposition~\ref{prop:decomp} reduces to the second claim of Proposition~\ref{prop:west}.
\begin{proposition}[Comparing two representations]\label{prop:compare}\StmtPropCompare\end{proposition}
\begin{proof}All expectations are finite and do not depend on the choice of versions, as in the proof of
Proposition~\ref{prop:decomp}. For bounded measurable $g$ on $\mathcal Z$,
$\E_{\Pt}[g(\phi(X))]=\E_P[w(X)g(\phi(X))]=\E_P[v_\phi(\phi(X))g(\phi(X))]$. Since $v_\phi(\phi(X))$ is integrable and
$\sigma(\phi(X))$-measurable, $\E_P[v_\phi(\phi(X))m]=\E_P[v_\phi(\phi(X))M_\phi]
=\E_P[v_\phi(\phi(X))\mu_\phi(\phi(X))]=\E_{\Pt}[\mu_\phi(\phi(X))]$, the last step by the previous identity with
$g=\mu_\phi$. Also $\E_P[wm]=\E_{\Pt}[m]=\E_{\Pt}[\tilde\mu_\phi(\phi(X))]$. Subtracting gives the first identity. It also
follows that $B_\phi=\E_P[v_\phi(\phi(X))M_\phi]-\E_P[wm]$, and likewise for $\psi$. Since $\sigma(\psi(X))\subseteq\sigma(\phi(X))$,
the tower property gives $\E_P[v_\phi(\phi(X))\mid\psi(X)]=\E_P[w(X)\mid\psi(X)]=v_\psi(\psi(X))$ and
$\E_P[M_\phi\mid\psi(X)]=M_\psi$ almost surely. Hence $B_\psi-B_\phi=\E_P[v_\psi(\psi(X))M_\psi]-\E_P[v_\phi(\phi(X))M_\phi]$,
which is the display. If $w(X)$ is $\sigma(\phi(X))$-measurable, then $v_\phi(\phi(X))=w(X)$ almost surely and
$B_\phi=0$.\end{proof}

\noindent By the first identity, $B_\phi$ is the second term of the limit in Proposition~\ref{prop:decomp}, the bias of
the representation $\phi$. The last expectation in the display above equals
$\E_P\bigl[\E_P[v_\phi(\phi(X))M_\phi\mid\psi(X)]-\E_P[v_\phi(\phi(X))\mid\psi(X)]\,\E_P[M_\phi\mid\psi(X)]\bigr]$, the
average under $P$ of the conditional covariance of $v_\phi(\phi(X))$ and $M_\phi$ given $\psi(X)$. It can have either sign. Let
$\mathcal X=\{a,b\}$ with $P_X$ uniform, $\phi$ the identity, $\psi$ constant, $w(a)=1+d$ and $w(b)=1-d$ with $|d|<1$,
distinct scores at each $x$, and a law of $Y$ given $X$ with $P(m(X,Y)=1\mid X=x)=\rho_x\in[0,1]$. Then the expectation
equals $\frac12d(\rho_a-\rho_b)$. Neither of $|B_\phi|$ and $|B_\psi|$ is the smaller in general: with
$\mathcal X=\{1,2,3,4\}$, $P_X$ uniform, $\phi(x)=\Ind\{x\le2\}$, $\psi$ constant, distinct scores at each $x$, and a law
of $Y$ given $X$ with $P(m(X,Y)=1\mid X=x)=\rho_x$, $w=(\frac32,\frac32,\frac12,\frac12)$ and $\rho=(\frac12,\frac12,0,0)$
give $0$ and $\frac18$, and $w=(2,1,\frac34,\frac14)$ and $\rho=(\frac25,0,1,\frac15)$ give $\frac1{10}$ and $0$.
All three examples arise with $\phi$ the penultimate embedding of a model and $\psi=G$ (Definition~\ref{def:GR}).
Take $K=2$, $T=1$ and $S(x,\cdot):=-\log\mathrm{softmax}(A\phi(x)+b)$, with $\phi(a)\ne\phi(b)$ in the first example and
$\phi(1)=\phi(2)\ne\phi(3)=\phi(4)$ in the other two, and choose $A$ and $b$ so that the logits at the two values of
$\phi$ are permutations of each other with distinct entries. The quantities above depend on $\phi$ only through
$\sigma(\phi(X))$, the scores at each $x$ are distinct, $G$ is constant, and $G=f\circ\phi$ with
$f(e):=\mathrm{sort}(-\log\mathrm{softmax}(Ae+b))$ continuous.
\subsection{Reduction to the sorted scores}
\paragraph{Definition~\ref{def:GR} (Sorted scores and rank), restated.} \StmtDefGr

\paragraph{Proposition~\ref{prop:reduction} (Reduction), restated.} \StmtPropReduction
\begin{proof}(i) For bounded measurable $\eta$,
$\E_{\Pt}[\eta(G,R)]=\E_{\Pt}[\sum_r\eta(G,r)q_r(G)]=\E_P[w\sum_r\eta(G,r)q_r(G)]=\E_P[v(G)\sum_r\eta(G,r)q_r(G)]
=\E_P[v(G)\eta(G,R)]$, by the assumption under $\Pt$, the change of measure, the definition of $v$ and the assumption
under $P$. Under the second condition of Proposition~\ref{prop:sufficient}, identity (i) is the conclusion of
\citep[Theorem~1]{stojanov2019low} with $(h(X),Y)$ replaced by $(G(X),R(X,Y))$.
(ii) Consider the reduced problem with covariate $Z:=G(X)\in[0,\infty)^K$, label $R:=R(X,Y)$ and the fixed score
$s(z,r):=z_r\ge0$. For every $(x,y)$, $S(x,y)=s(G(x),R(x,y))$, $\tau(x)=G(x)_{T+1}$,
$m(x,y)=\Ind\{s(G(x),R(x,y))\ge G(x)_{T+1}\}$ and $h(x,y)=G(x)_{T+1}m(x,y)$; these are the quantities of
Definition~\ref{def:transform} for the score $s$ at $(G(x),R(x,y))$. Since $y\mapsto R(X_{n+1},y)$ is a bijection,
the vector of e-values of the original procedure with weights $v(G(X_j))$ is a permutation of the vector of the
reduced procedure, $F$ depends only on the multiset of entries, and $Y_{n+1}\notin\hat C$ iff $R(X_{n+1},Y_{n+1})$ is
not in the reduced set; likewise for the transformed score. The reduced calibration data $(Z_i,R_i)_{i\le n}$ are
i.i.d.\ from the law of $(G(X),R(X,Y))$ under $P$; the reduced test point is a function of $(X_{n+1},Y_{n+1})$,
hence independent of them, and by (i) its law has density $v(z)$ with respect to the calibration law. So the reduced
problem satisfies Definition~\ref{def:setting} with density ratio $v>0$ and score $s$, and
Theorems~\ref{thm:validity}, \ref{thm:posthoc}, \ref{thm:transform}, \ref{thm:corrected} and
Proposition~\ref{prop:sqrt} applied to it give the stated inequalities.\end{proof}

\begin{proposition}[Conditions for the reduction]\label{prop:sufficient}\StmtPropSufficient\end{proposition}
\begin{proof}By the factorization lemma there are measurable $q_r:[0,\infty)^K\to[0,1]$ with $P(R=r\mid G)=q_r(G)$ $P$-a.s. It
remains to show $\E_{\Pt}[\zeta(G)\Ind\{R=r\}]=\E_{\Pt}[\zeta(G)q_r(G)]$ for every $r$ and every bounded measurable
$\zeta$ on $[0,\infty)^K$. For bounded measurable $g$, $\E_{\Pt}[g(G)]=\E_P[w\,g(G)]=\E_P[v(G)g(G)]$. If
$w(X)=v(G(X))$ $P$-a.s., then
$\E_{\Pt}[\zeta(G)\Ind\{R=r\}]=\E_P[w\zeta(G)\Ind\{R=r\}]=\E_P[v(G)\zeta(G)\Ind\{R=r\}]=\E_P[v(G)\zeta(G)q_r(G)]
=\E_{\Pt}[\zeta(G)q_r(G)]$, conditioning on $G$ in the third step. If $P(R=r\mid X)=P(R=r\mid G(X))$ $P$-a.s., then
$P(R=r\mid X)=q_r(G)$ $P$-a.s.\ and
$\E_{\Pt}[\zeta(G)\Ind\{R=r\}]=\E_P[w\zeta(G)\Ind\{R=r\}]=\E_P[w\zeta(G)q_r(G)]=\E_{\Pt}[\zeta(G)q_r(G)]$,
conditioning on $X$ in the second step. This second condition is condition (i) of \citep[Theorem~1]{stojanov2019low} with
$(h,Y)$ replaced by $(G,R)$. Under the third condition, its instances $\eta(g,r')=\zeta(g)\Ind\{r'=r\}$ and
$\eta(g,r')=\zeta(g)q_r(g)$, with conditioning on $G$ under $P$ in between, give
$\E_{\Pt}[\zeta(G)\Ind\{R=r\}]=\E_P[v(G)\zeta(G)\Ind\{R=r\}]=\E_P[v(G)\zeta(G)q_r(G)]=\E_{\Pt}[\zeta(G)q_r(G)]$.\end{proof}

\noindent By the third condition and part (i) of Proposition~\ref{prop:reduction}, the assumption of
Proposition~\ref{prop:reduction} holds exactly when the law of $(G(X),R(X,Y))$ under $\Pt$ has density
$(g,r)\mapsto v(g)$ with respect to its law under $P$, that is, when the density ratio of the sorted scores and the
rank depends only on the sorted scores.
\begin{remark}Assume that the version $v$ of Definition~\ref{def:GR} satisfies $v\le\omega_{\max}$ for a constant
$\omega_{\max}<\infty$. Let $\omega_i:=v(G(X_i))$, and assume $\mathbb P(\tilde b_j(\omega)=m_j\ \forall j\le n)\to1$. In
Proposition~\ref{prop:decomp} with $\hat\omega=v\circ G$ and $\phi=G$, $\tilde\omega=v\circ G$ (since
$\E_P[v(G(X))]=\E_P[w(X)]=1$), and we may take $v_G=v$, so the first term of the limit vanishes and the asymptotic bias
of the tight report with weights $v\circ G$ is $-\E_{\Pt}[\tilde\mu_G(G(X))-\mu_G(G(X))]$.
The absolute value of this bias is bounded, for a bounded loss, by a total variation distance
\citep[Appendix~B.3, Lemma~1]{chen2021mandoline}; here it is given exactly for the loss $m$. Since $m(x,y)=\Ind\{G(x)_{R(x,y)}\ge G(x)_{T+1}\}$, under the
assumption of Proposition~\ref{prop:reduction} the conditional law of $m$ given $G$ is the same under $P$ and $\Pt$, so
$\tilde\mu_G(G(X))=\mu_G(G(X))$ $\Pt$-a.s.\ and the asymptotic bias is zero.\end{remark}

\section{Why the test score is not reweighted}\label{app:design}

\begin{remark}
A seemingly natural alternative also reweights the test score in the numerator,
$\bar E(y):=W\,w_{n+1}S(X_{n+1},y)/\bigl(\sum_{i\le n}w_iS_i+w_{n+1}S(X_{n+1},y)\bigr)$. Repeating the change of
measure of Theorem~\ref{thm:validity} leaves an extra factor $w_{n+1}$, and by the law of large numbers
$\E_{\Pt}[\bar E(Y_{n+1})]\to\E_{\Pt}[w(X)S(X,Y)]/\E_{\Pt}[S(X,Y)]$ as $n\to\infty$ for bounded weights and scores, which
exceeds $1$ when the shift moves mass toward inputs with larger scores. For example, with $P(A)=1/3$, $w=2$ on $A$ and
$w=1/2$ on $A^c$, and $S=10$ on $A$ and $S=1$ on $A^c$, the limit is $27/14$.
\end{remark}

\section{Experimental details}\label{app:details}

\paragraph{Domain classifier.} Unless stated otherwise, the domain classifier in each draw is scikit-learn's
\texttt{HistGradientBoostingClassifier} \citep{pedregosa2011scikit} with maximum depth $3$ and $200$ iterations, trained
to separate the clean sample from the unlabeled test sample. With its predicted probability $\hat q$ of the test class
clipped to $[0.01,0.99]$ as in \citet{tibshirani2019conformal}, the weights are $\hat\omega=\hat q/(1-\hat q)$; a
constant factor, such as the ratio of the two sample sizes, cancels in the reports. The clip bounds the weights, as
Proposition~\ref{prop:decomp} assumes. Appendix~\ref{app:extra} gives the results under looser clips.

\paragraph{Dry Bean.} The data \citep{koklu2020multiclass} have $13{,}611$ beans with $16$ features and $7$ classes. They are split with scikit-learn's \texttt{train\_test\_split}
(test size $0.6$, stratified, seed $42$) into $5{,}444$ training beans and a pool of $8{,}167$. The classifier is XGBoost
with $15$ trees of depth $2$ and multiclass log-loss; its accuracy on the pool is $0.913$. The score is $-\log$ of the predicted probability, clipped at
$10^{-12}$. Each of the $80$ draws (seed $1$) takes $500$ calibration beans and $500$ disjoint clean beans for the domain
classifier; the rest of the pool is shuffled and split in halves, and $1{,}000$ unlabeled test beans and $500$ labeled
test beans are drawn with replacement from the two halves with probability proportional to $\exp(\beta t(x))$. The
truth is the exact miscoverage of the size-$T$ set under the tilted pool,
$\sum_xe^{\beta t(x)}m(x,y_x)/\sum_xe^{\beta t(x)}$ over the $8{,}167$ pool beans $x$ with labels $y_x$; the $500$
labeled test beans are used only to evaluate WSCP. Computed on the half from which the unlabeled test beans of a draw
are drawn, this miscoverage differs from the pool value by at most \Ndbtruthhalfmaxdev{} on average over the $80$
draws. Feature weights use standardized features.
The calibrated variants use scikit-learn's \texttt{CalibratedClassifierCV} (sigmoid or isotonic, with one final model or
the average of five fold models) with folds grouped by bean, so that repeated draws of the same bean never appear in
both a training and a held-out fold. Early stopping chooses the number of boosting iterations that minimizes log-loss
on a grouped $20\%$ hold-out. Logistic regression is scikit-learn's \texttt{LogisticRegression} with default
settings and at most $1{,}000$ iterations, on the standardized features. ECE uses $15$ equal-width bins on a held-out domain sample of $500$ clean and $1{,}000$
test beans.

\paragraph{CIFAR-10-C.} Features are the logits and the input to the last linear layer of each model ($320$ dimensions for EfficientNet-B0, $128$ for the WideResNet-40-2), computed once on
the $10{,}000$ clean test images and on each corrupted set. The EfficientNet-B0 weights are those of the public code
release of \citet{gauthier2025backward} (clean accuracy $91.09\%$); we did not find a stated license. The WideResNet-40-2 is RobustBench's
\texttt{Hendrycks2020AugMix\_WRN} (RobustBench code and model weights: MIT license; AugMix code: Apache 2.0); its clean accuracy
($95.08\%$) and mean accuracy over the $75$ corrupted sets ($88.82\%$) match the RobustBench leaderboard. CIFAR-10-C is
distributed under CC BY 4.0, and Dry Bean under CC BY 4.0. Image $i$ of every corrupted set is the corrupted version of clean image $i$. Each cell uses $60$ draws, with the
same images in every cell. A draw takes $1{,}000$ calibration images and $1{,}000$ further clean images for the domain
classifier (seed $0$) and then, from the corrupted versions of the other $8{,}000$ images (seed $1$), $2{,}000$ unlabeled
images for the domain classifier and ATC-NE, so no corrupted image
of a draw is a version of one of its clean images. The wall criterion is
the AUC of a gradient-boosting classifier trained on $2{,}000$ clean and $2{,}000$ corrupted embeddings and evaluated
on $2{,}000$ other clean and $2{,}000$ other corrupted embeddings, the four samples being versions of different images
(seed $7$). The sorted-score AUC uses the same procedure and the same samples on $G(x)$. At the walls the embedding AUC is at least \Ncifwallembaucmin{}
and the sorted-score AUC is between \Ncifwallscoreaucmin{} and \Ncifwallscoreaucmax. The three walls of EfficientNet-B0 are its three cells with the
largest top-1 error, but for the WideResNet contrast at severity $3$ is a wall (AUC \Nwallexauc) with a top-1 error of
\Nwallextopone\%, below that of JPEG compression at severities $3$ and $5$ (\Nwallexjpegthree\% and \Nwallexjpegfive\%),
which are not walls. Weighted split conformal prediction chooses its level on the $2{,}000$
unlabeled corrupted images used for the weights and is evaluated on the corrupted versions of the $6{,}000$ images that are neither in
that sample nor among the clean images of the draw.

\paragraph{CIFAR-10.1.} We use version v6 \citep{recht2018cifar}, with $2{,}000$ images, $200$ per class. Its images are
part of the Tiny Images data set, and its repository states that they can be used in the same way; the code of the
repository is under the MIT license. Features are computed as for CIFAR-10-C. Each of the $60$ draws (seed $0$) takes
$1{,}000$ calibration images and $1{,}000$ further images for the domain classifier from the clean CIFAR-10 test images,
and $1{,}000$ unlabeled CIFAR-10.1 images for the domain classifier and ATC-NE. The truth is the miss rate of the
size-$T$ set over all $2{,}000$ images. A gradient-boosting classifier fitted on $1{,}000$ clean and $1{,}000$ CIFAR-10.1
images separates $1{,}000$ others of each with an AUC of at most \Nnataucmax{} on the embedding or the sorted scores
(seed $7$), so there is no wall.

\paragraph{Weighted split conformal prediction.} We match the mean set size rather than cap every set at $T$, because a
WSCP set with at most $T$ labels lies inside the size-$T$ set, so a capped WSCP could not miss less often. Its level is
chosen after seeing the calibration scores, so the guarantee of WSCP, which is for a level fixed in advance, does not
strictly apply. Its miscoverage is measured on held-out labeled test points. Appendix~\ref{app:extra} repeats it with a
level fixed in advance.

\paragraph{Compute.} The WideResNet features were extracted on one NVIDIA V100 GPU in about two minutes; the
EfficientNet-B0 features were computed once from the released weights. All other computations ran on CPUs. The
CIFAR-10-C reports and WSCP for one model and corruption took $4$ to $7$ minutes as one job on $4$ cores, and the
checks with other domain classifiers at most $30$ minutes.

\section{Additional results}\label{app:extra}

\paragraph{The errors are understatements.} Table~\ref{tab:signed} gives the mean gaps behind the mean absolute gaps of
Table~\ref{tab:main}. They are negative for every estimator in every group, except for the tight report with score
weights in one group. Of the \Ncifnsettings{} settings, the tight report understates the miscoverage in
\Ncifsignednoneunder{} without weights, \Ncifsignedembunder{} with embedding weights and \Ncifsignedscoreunder{} with
score weights, and ATC-NE understates it in \Ncifsignedatcunder.

\begin{table}[htbp]
\centering\small
\caption{CIFAR-10-C by model, $T$ and overlap, as Table~\ref{tab:main}: mean gap (report $-$ truth) in percentage points
over the settings of a group and their $60$ draws, for the tight report with no, embedding and score weights and for
ATC-NE.}\label{tab:signed}
\begin{tabular}{lllcccc}
\toprule
 & & & \multicolumn{3}{c}{tight report} & \\
\cmidrule(lr){4-6}
model & $T$ & & none & emb. & score & ATC-NE \\
\midrule
EfficientNet-B0 & 1 & no wall & $-$14.64 & $-$7.86 & $-$3.05 & $-$2.59 \\
 &  & walls & $-$68.13 & $-$52.33 & $-$42.98 & $-$29.59 \\
\addlinespace
 & 2 & no wall & $-$9.15 & $-$6.36 & $-$3.43 & $-$2.79 \\
 &  & walls & $-$60.80 & $-$53.65 & $-$48.67 & $-$41.87 \\
\addlinespace
 & 3 & no wall & $-$5.64 & $-$4.03 & $-$2.27 & $-$1.88 \\
 &  & walls & $-$49.08 & $-$43.90 & $-$42.39 & $-$38.66 \\
\midrule
WRN-40-2 (AugMix) & 1 & no wall & $-$3.00 & $-$0.86 & +0.19 & $-$0.12 \\
 &  & walls & $-$12.53 & $-$5.03 & $-$5.26 & $-$3.68 \\
\addlinespace
 & 2 & no wall & $-$1.46 & $-$0.86 & $-$0.36 & $-$0.29 \\
 &  & walls & $-$6.80 & $-$4.94 & $-$4.23 & $-$3.31 \\
\addlinespace
 & 3 & no wall & $-$0.69 & $-$0.44 & $-$0.19 & $-$0.02 \\
 &  & walls & $-$3.90 & $-$3.21 & $-$2.65 & $-$2.13 \\
\bottomrule
\end{tabular}

\end{table}

\begin{table}[htbp]
\centering\small
\caption{CIFAR-10-C: number of cells in which each method is valid (mean over $60$ draws of the report minus the true
miscoverage is $\ge0$; for WSCP, mean miscoverage $\le$ its level $\alpha$), for the identity report with score weights
(``score''), embedding weights (``emb.'') and no weights (``none''). WSCP uses score weights; $|C|>T$ is the mean share
of test images whose WSCP set has more than $T$ labels.}\label{tab:cifar}
\resizebox{\linewidth}{!}{\begin{tabular}{llccccccccccc}
\toprule
 & & \multicolumn{6}{c}{settings without a wall (embedding AUC $<0.95$)} & \multicolumn{5}{c}{walls (AUC $\ge0.95$)} \\
\cmidrule(lr){3-8}\cmidrule(lr){9-13}
model & $T$ & \#cells & \multicolumn{3}{c}{identity valid} & \multicolumn{2}{c}{WSCP} & \#cells & \multicolumn{3}{c}{identity valid} & WSCP \\
 & & & score & emb. & none & valid & $|C|>T$ & & score & emb. & none & valid \\
\midrule
 & 1 & 12 & 9 & 2 & 0 & 7 & 2\% & 3 & 0 & 0 & 0 & 0 \\
EfficientNet-B0 & 2 & 12 & 10 & 7 & 3 & 1 & 25\% & 3 & 0 & 0 & 0 & 0 \\
 & 3 & 12 & 11 & 9 & 6 & 2 & 30\% & 3 & 0 & 0 & 0 & 0 \\
\midrule
 & 1 & 9 & 0 & 0 & 0 & 8 & 1\% & 6 & 0 & 2 & 0 & 1 \\
WRN-40-2 (AugMix) & 2 & 9 & 9 & 8 & 5 & 4 & 19\% & 6 & 4 & 6 & 1 & 0 \\
 & 3 & 9 & 9 & 9 & 7 & 4 & 23\% & 6 & 5 & 6 & 2 & 0 \\
\bottomrule
\end{tabular}
}
\end{table}

\begin{figure}[htbp]
\centering
\includegraphics[width=\linewidth]{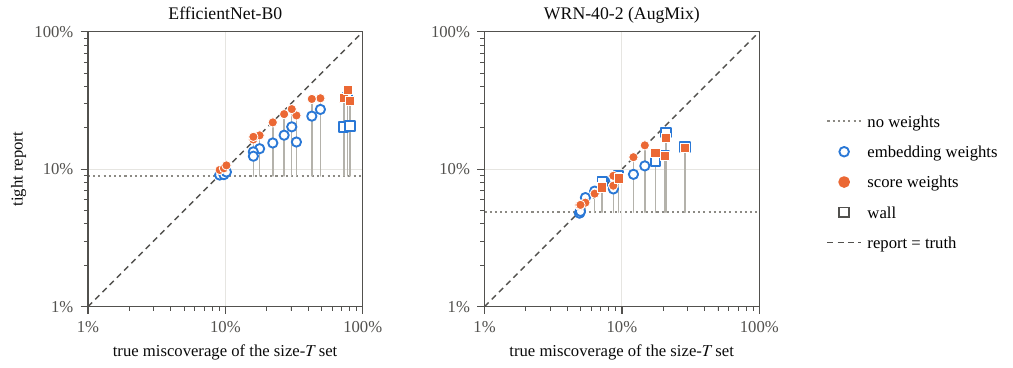}
\includegraphics[width=\linewidth]{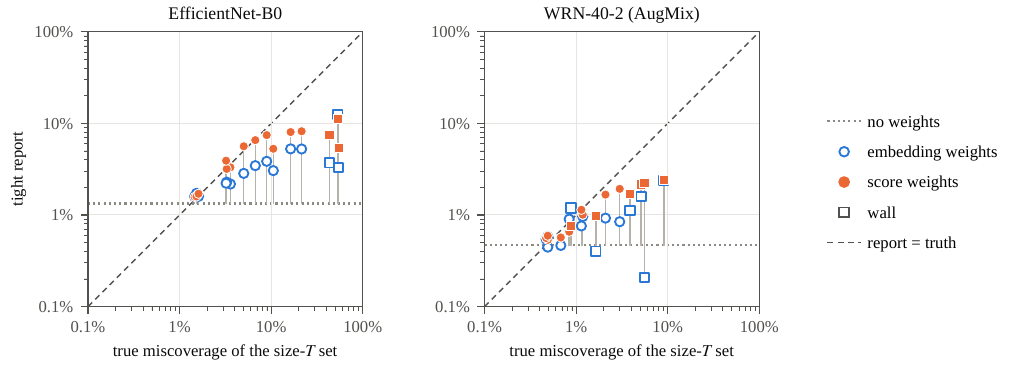}
\caption{CIFAR-10-C at $T=1$ (top) and $T=3$ (bottom), with the layout of Figure~\ref{fig:cifar}.}\label{fig:cifarT13}
\end{figure}

\begin{table}[htbp]
\centering\small
\caption{Dry Bean, $T=2$, same layout as Table~\ref{tab:drybean}. Largest standard error \NdbTTwomaxse.}\label{tab:drybeantwo}
\resizebox{\linewidth}{!}{\begin{tabular}{lcccccc}
\toprule
weights & hardness $\beta=1$ & hardness $\beta=2$ & easy $\beta=1$ & easy $\beta=2$ & benign $\beta=1$ & benign $\beta=2$ \\
\midrule
true miscoverage & 0.064 & 0.144 & 0.005 & 0.003 & 0.012 & 0.007 \\
\midrule
\multicolumn{7}{l}{\emph{identity report}} \\
true weights & +0.170 & +0.316 & +0.025 & +0.018 & +0.043 & +0.040 \\
features & +0.094 & +0.152 & +0.035 & +0.028 & +0.045 & +0.044 \\
features + Platt & +0.055 & +0.110 & +0.051 & +0.036 & +0.050 & +0.046 \\
sorted scores & +0.233 & +0.287 & +0.026 & +0.021 & +0.046 & +0.037 \\
\multicolumn{7}{l}{\emph{tight report}} \\
true weights & $-$0.003 & $-$0.021 & 0.000 & 0.000 & 0.000 & 0.000 \\
features & $-$0.027 & $-$0.061 & +0.002 & +0.002 & $-$0.001 & 0.000 \\
features + Platt & $-$0.037 & $-$0.074 & +0.007 & +0.005 & +0.001 & +0.002 \\
sorted scores & +0.015 & $-$0.036 & $-$0.001 & $-$0.001 & $-$0.001 & $-$0.001 \\
\midrule
ATC-NE & +0.120 & +0.384 & $-$0.004 & $-$0.003 & $-$0.003 & $-$0.004 \\
\bottomrule
\end{tabular}
}
\end{table}

\paragraph{ATC on Dry Bean.} As on CIFAR-10-C, the threshold of ATC-NE is the confidence value below which the share of
calibration points equals their miss rate. Here it is fitted on the $500$ calibration beans of a draw, and the estimate
is the share of the $1{,}000$ unlabeled test beans below it; its gaps are the last rows of Table~\ref{tab:drybean} and
Table~\ref{tab:drybeantwo}. Under the hardness shift ATC-NE overstates the miscoverage in all four settings, and under the easy and
benign shifts it understates it. The hardness tilt is the model's entropy, the negative of the confidence that ATC-NE thresholds, so this
shift moves weight toward low confidence. ATC counts every bean below its threshold as a miss, but in the pool only
\Ndbatcbelowmiss\% of the beans below the threshold at $T=1$ are misses, so such a shift is overcounted. The overstatement
remains under other choices. At $T=1$ the gaps under the hardness shift are \NdbhardnessOneTOneATCMC{} and
\NdbhardnessTwoTOneATCMC{} with the maximum probability as the confidence, \NdbhardnessOneTOneATCTS{} and
\NdbhardnessTwoTOneATCTS{} when the probabilities are first calibrated by temperature scaling on the calibration beans,
and \NdbhardnessOneTOneATCcode{} and \NdbhardnessTwoTOneATCcode{} with the threshold of the authors' code, which is the next
lower confidence value among the calibration points.

For each CIFAR-10-C cell the tables below give the embedding AUC ($^*$: wall), the true miscoverage of the size-$T$
set, the mean gap over $60$ draws (standard error in parentheses) of the identity, transformed ($\hat\alpha^{\rm tr}$)
and tight reports with score weights, of the identity and tight reports with embedding weights, and of ATC-NE, and, for
WSCP with score weights, its level $\alpha$, its mean miscoverage and the mean share of test images whose set has more
than $T$ labels. A gap that rounds to $0.000$ keeps its sign.

\begin{table}[htbp]
\centering\scriptsize
\caption{CIFAR-10-C, EfficientNet-B0, $T=1$.}
\resizebox{\linewidth}{!}{\begin{tabular}{lcc|ccc|cc|c|ccc}
\toprule
cell & AUC & truth & \multicolumn{3}{c|}{score weights: gap} & \multicolumn{2}{c|}{emb.\ weights: gap} & ATC-NE & \multicolumn{3}{c}{WSCP (score weights)} \\
 & & & identity & transf. & tight & identity & tight & gap & $\alpha$ & miscov. & $|C|>T$ \\
\midrule
Gaussian noise 1 & 0.909 & 0.329 & $-$0.013\,(0.007) & $-$0.125\,(0.004) & $-$0.083\,(0.005) & $-$0.159\,(0.007) & $-$0.171\,(0.006) & $-$0.073\,(0.003) & 0.256 & 0.327 & 2\% \\
Gaussian noise 3 & 0.985$^*$ & 0.727 & $-$0.230\,(0.008) & $-$0.407\,(0.006) & $-$0.397\,(0.006) & $-$0.477\,(0.012) & $-$0.525\,(0.010) & $-$0.336\,(0.004) & 0.343 & 0.728 & 3\% \\
Gaussian noise 5 & 0.993$^*$ & 0.809 & $-$0.318\,(0.007) & $-$0.497\,(0.005) & $-$0.493\,(0.004) & $-$0.541\,(0.013) & $-$0.603\,(0.010) & $-$0.438\,(0.004) & 0.328 & 0.808 & 3\% \\
Defocus blur 1 & 0.513 & 0.091 & +0.001\,(0.003) & $-$0.032\,(0.001) & +0.008\,(0.002) & +0.003\,(0.009) & $-$0.000\,(0.002) & $-$0.001\,(0.001) & 0.102 & 0.092 & 1\% \\
Defocus blur 3 & 0.769 & 0.177 & +0.032\,(0.006) & $-$0.046\,(0.003) & $-$0.001\,(0.003) & $-$0.031\,(0.004) & $-$0.037\,(0.003) & $-$0.014\,(0.002) & 0.182 & 0.177 & 1\% \\
Defocus blur 5 & 0.941 & 0.492 & +0.014\,(0.009) & $-$0.174\,(0.006) & $-$0.164\,(0.006) & $-$0.117\,(0.010) & $-$0.220\,(0.007) & $-$0.113\,(0.004) & 0.340 & 0.492 & 3\% \\
Fog 1 & 0.522 & 0.098 & $-$0.005\,(0.003) & $-$0.036\,(0.002) & +0.004\,(0.002) & +0.000\,(0.010) & $-$0.006\,(0.002) & $-$0.003\,(0.001) & 0.104 & 0.097 & 1\% \\
Fog 3 & 0.733 & 0.160 & +0.017\,(0.004) & $-$0.044\,(0.003) & +0.005\,(0.003) & $-$0.025\,(0.007) & $-$0.026\,(0.003) & $-$0.003\,(0.002) & 0.169 & 0.160 & 1\% \\
Fog 5 & 0.915 & 0.426 & +0.049\,(0.008) & $-$0.122\,(0.006) & $-$0.101\,(0.006) & $-$0.119\,(0.008) & $-$0.183\,(0.007) & $-$0.069\,(0.003) & 0.338 & 0.424 & 3\% \\
Contrast 1 & 0.596 & 0.102 & $-$0.006\,(0.002) & $-$0.037\,(0.001) & +0.005\,(0.002) & $-$0.007\,(0.008) & $-$0.006\,(0.002) & $-$0.005\,(0.001) & 0.109 & 0.103 & 1\% \\
Contrast 3 & 0.912 & 0.304 & +0.069\,(0.008) & $-$0.067\,(0.005) & $-$0.031\,(0.006) & $-$0.067\,(0.008) & $-$0.101\,(0.006) & $-$0.019\,(0.003) & 0.282 & 0.304 & 2\% \\
Contrast 5 & 1.000$^*$ & 0.775 & $-$0.114\,(0.011) & $-$0.355\,(0.009) & $-$0.399\,(0.009) & $-$0.249\,(0.025) & $-$0.442\,(0.025) & $-$0.113\,(0.005) & 0.410 & 0.776 & 6\% \\
JPEG 1 & 0.697 & 0.160 & +0.029\,(0.006) & $-$0.038\,(0.003) & +0.012\,(0.004) & $-$0.028\,(0.007) & $-$0.036\,(0.002) & $-$0.000\,(0.002) & 0.179 & 0.159 & 1\% \\
JPEG 3 & 0.763 & 0.221 & +0.053\,(0.008) & $-$0.048\,(0.005) & $-$0.002\,(0.006) & $-$0.055\,(0.004) & $-$0.066\,(0.002) & $-$0.002\,(0.002) & 0.227 & 0.220 & 2\% \\
JPEG 5 & 0.800 & 0.268 & +0.066\,(0.009) & $-$0.058\,(0.005) & $-$0.016\,(0.006) & $-$0.064\,(0.005) & $-$0.091\,(0.003) & $-$0.009\,(0.003) & 0.264 & 0.265 & 2\% \\
\bottomrule
\end{tabular}
}
\end{table}
\begin{table}[htbp]
\centering\scriptsize
\caption{CIFAR-10-C, EfficientNet-B0, $T=2$.}
\resizebox{\linewidth}{!}{\begin{tabular}{lcc|ccc|cc|c|ccc}
\toprule
cell & AUC & truth & \multicolumn{3}{c|}{score weights: gap} & \multicolumn{2}{c|}{emb.\ weights: gap} & ATC-NE & \multicolumn{3}{c}{WSCP (score weights)} \\
 & & & identity & transf. & tight & identity & tight & gap & $\alpha$ & miscov. & $|C|>T$ \\
\midrule
Gaussian noise 1 & 0.909 & 0.179 & $-$0.000\,(0.004) & $-$0.104\,(0.003) & $-$0.074\,(0.005) & $-$0.076\,(0.004) & $-$0.122\,(0.003) & $-$0.067\,(0.003) & 0.043 & 0.141 & 26\% \\
Gaussian noise 3 & 0.985$^*$ & 0.577 & $-$0.295\,(0.007) & $-$0.457\,(0.004) & $-$0.430\,(0.006) & $-$0.421\,(0.007) & $-$0.504\,(0.006) & $-$0.398\,(0.005) & 0.087 & 0.567 & 27\% \\
Gaussian noise 5 & 0.993$^*$ & 0.690 & $-$0.437\,(0.005) & $-$0.593\,(0.003) & $-$0.572\,(0.004) & $-$0.522\,(0.007) & $-$0.621\,(0.007) & $-$0.547\,(0.004) & 0.076 & 0.682 & 25\% \\
Defocus blur 1 & 0.513 & 0.032 & +0.023\,(0.002) & $-$0.006\,(0.001) & +0.003\,(0.001) & +0.026\,(0.006) & +0.001\,(0.002) & +0.001\,(0.001) & 0.010 & 0.010 & 23\% \\
Defocus blur 3 & 0.769 & 0.079 & +0.038\,(0.004) & $-$0.026\,(0.003) & $-$0.011\,(0.003) & +0.003\,(0.002) & $-$0.032\,(0.002) & $-$0.013\,(0.002) & 0.025 & 0.040 & 25\% \\
Defocus blur 5 & 0.941 & 0.334 & $-$0.053\,(0.007) & $-$0.212\,(0.005) & $-$0.188\,(0.006) & $-$0.128\,(0.007) & $-$0.231\,(0.005) & $-$0.152\,(0.004) & 0.087 & 0.303 & 25\% \\
Fog 1 & 0.522 & 0.032 & +0.022\,(0.001) & $-$0.007\,(0.001) & +0.004\,(0.001) & +0.029\,(0.006) & +0.002\,(0.002) & +0.002\,(0.001) & 0.010 & 0.010 & 23\% \\
Fog 3 & 0.733 & 0.066 & +0.034\,(0.003) & $-$0.023\,(0.002) & $-$0.002\,(0.003) & +0.013\,(0.005) & $-$0.021\,(0.002) & $-$0.005\,(0.002) & 0.021 & 0.030 & 25\% \\
Fog 5 & 0.915 & 0.263 & +0.009\,(0.006) & $-$0.146\,(0.004) & $-$0.118\,(0.006) & $-$0.089\,(0.005) & $-$0.170\,(0.004) & $-$0.088\,(0.004) & 0.081 & 0.234 & 25\% \\
Contrast 1 & 0.596 & 0.035 & +0.021\,(0.001) & $-$0.011\,(0.001) & +0.001\,(0.001) & +0.022\,(0.005) & $-$0.003\,(0.002) & $-$0.001\,(0.001) & 0.010 & 0.011 & 23\% \\
Contrast 3 & 0.912 & 0.153 & +0.061\,(0.006) & $-$0.057\,(0.004) & $-$0.024\,(0.006) & $-$0.013\,(0.005) & $-$0.076\,(0.004) & $-$0.015\,(0.003) & 0.057 & 0.121 & 25\% \\
Contrast 5 & 1.000$^*$ & 0.647 & $-$0.257\,(0.010) & $-$0.463\,(0.007) & $-$0.458\,(0.008) & $-$0.254\,(0.019) & $-$0.485\,(0.020) & $-$0.312\,(0.007) & 0.195 & 0.647 & 22\% \\
JPEG 1 & 0.697 & 0.064 & +0.046\,(0.004) & $-$0.018\,(0.002) & +0.008\,(0.004) & +0.015\,(0.004) & $-$0.017\,(0.002) & +0.006\,(0.002) & 0.022 & 0.030 & 25\% \\
JPEG 3 & 0.763 & 0.095 & +0.062\,(0.005) & $-$0.027\,(0.003) & +0.004\,(0.005) & +0.001\,(0.002) & $-$0.036\,(0.002) & +0.006\,(0.002) & 0.035 & 0.061 & 24\% \\
JPEG 5 & 0.800 & 0.128 & +0.063\,(0.006) & $-$0.045\,(0.004) & $-$0.014\,(0.006) & $-$0.012\,(0.003) & $-$0.059\,(0.002) & $-$0.009\,(0.002) & 0.046 & 0.092 & 25\% \\
\bottomrule
\end{tabular}
}
\end{table}
\begin{table}[htbp]
\centering\scriptsize
\caption{CIFAR-10-C, EfficientNet-B0, $T=3$.}
\resizebox{\linewidth}{!}{\begin{tabular}{lcc|ccc|cc|c|ccc}
\toprule
cell & AUC & truth & \multicolumn{3}{c|}{score weights: gap} & \multicolumn{2}{c|}{emb.\ weights: gap} & ATC-NE & \multicolumn{3}{c}{WSCP (score weights)} \\
 & & & identity & transf. & tight & identity & tight & gap & $\alpha$ & miscov. & $|C|>T$ \\
\midrule
Gaussian noise 1 & 0.909 & 0.105 & +0.028\,(0.003) & $-$0.067\,(0.002) & $-$0.053\,(0.003) & $-$0.023\,(0.003) & $-$0.075\,(0.002) & $-$0.048\,(0.002) & 0.016 & 0.075 & 32\% \\
Gaussian noise 3 & 0.985$^*$ & 0.431 & $-$0.228\,(0.005) & $-$0.370\,(0.003) & $-$0.356\,(0.005) & $-$0.305\,(0.005) & $-$0.394\,(0.004) & $-$0.340\,(0.004) & 0.035 & 0.442 & 31\% \\
Gaussian noise 5 & 0.993$^*$ & 0.546 & $-$0.370\,(0.003) & $-$0.500\,(0.002) & $-$0.493\,(0.003) & $-$0.408\,(0.006) & $-$0.513\,(0.005) & $-$0.482\,(0.003) & 0.030 & 0.547 & 31\% \\
Defocus blur 1 & 0.513 & 0.015 & +0.029\,(0.001) & $-$0.002\,(0.001) & +0.001\,(0.001) & +0.033\,(0.005) & +0.001\,(0.002) & +0.000\,(0.001) & 0.006 & 0.005 & 28\% \\
Defocus blur 3 & 0.769 & 0.036 & +0.053\,(0.003) & $-$0.009\,(0.002) & $-$0.003\,(0.002) & +0.029\,(0.002) & $-$0.014\,(0.001) & $-$0.003\,(0.001) & 0.011 & 0.018 & 30\% \\
Defocus blur 5 & 0.941 & 0.215 & $-$0.009\,(0.005) & $-$0.148\,(0.003) & $-$0.133\,(0.005) & $-$0.059\,(0.005) & $-$0.163\,(0.004) & $-$0.114\,(0.003) & 0.032 & 0.205 & 30\% \\
Fog 1 & 0.522 & 0.015 & +0.028\,(0.001) & $-$0.003\,(0.001) & +0.001\,(0.001) & +0.034\,(0.005) & +0.002\,(0.002) & $-$0.000\,(0.001) & 0.005 & 0.005 & 28\% \\
Fog 3 & 0.733 & 0.033 & +0.045\,(0.002) & $-$0.011\,(0.001) & $-$0.001\,(0.002) & +0.031\,(0.004) & $-$0.010\,(0.001) & $-$0.001\,(0.001) & 0.008 & 0.013 & 30\% \\
Fog 5 & 0.915 & 0.163 & +0.036\,(0.005) & $-$0.099\,(0.003) & $-$0.083\,(0.005) & $-$0.030\,(0.003) & $-$0.110\,(0.003) & $-$0.059\,(0.003) & 0.030 & 0.141 & 30\% \\
Contrast 1 & 0.596 & 0.016 & +0.029\,(0.001) & $-$0.003\,(0.001) & +0.001\,(0.001) & +0.031\,(0.004) & $-$0.000\,(0.001) & $-$0.000\,(0.001) & 0.005 & 0.005 & 29\% \\
Contrast 3 & 0.912 & 0.089 & +0.073\,(0.004) & $-$0.035\,(0.003) & $-$0.015\,(0.004) & +0.020\,(0.004) & $-$0.051\,(0.003) & $-$0.012\,(0.003) & 0.020 & 0.060 & 30\% \\
Contrast 5 & 1.000$^*$ & 0.535 & $-$0.236\,(0.007) & $-$0.419\,(0.005) & $-$0.423\,(0.007) & $-$0.182\,(0.017) & $-$0.410\,(0.017) & $-$0.339\,(0.007) & 0.091 & 0.536 & 32\% \\
JPEG 1 & 0.697 & 0.032 & +0.052\,(0.003) & $-$0.007\,(0.002) & +0.007\,(0.003) & +0.031\,(0.004) & $-$0.010\,(0.001) & +0.005\,(0.001) & 0.008 & 0.014 & 29\% \\
JPEG 3 & 0.763 & 0.050 & +0.071\,(0.004) & $-$0.012\,(0.002) & +0.006\,(0.004) & +0.026\,(0.002) & $-$0.022\,(0.001) & +0.006\,(0.002) & 0.012 & 0.027 & 30\% \\
JPEG 5 & 0.800 & 0.067 & +0.078\,(0.005) & $-$0.020\,(0.003) & $-$0.001\,(0.005) & +0.024\,(0.002) & $-$0.033\,(0.002) & +0.001\,(0.002) & 0.015 & 0.042 & 30\% \\
\bottomrule
\end{tabular}
}
\end{table}
\begin{table}[htbp]
\centering\scriptsize
\caption{CIFAR-10-C, WideResNet-40-2 (AugMix), $T=1$.}
\resizebox{\linewidth}{!}{\begin{tabular}{lcc|ccc|cc|c|ccc}
\toprule
cell & AUC & truth & \multicolumn{3}{c|}{score weights: gap} & \multicolumn{2}{c|}{emb.\ weights: gap} & ATC-NE & \multicolumn{3}{c}{WSCP (score weights)} \\
 & & & identity & transf. & tight & identity & tight & gap & $\alpha$ & miscov. & $|C|>T$ \\
\midrule
Gaussian noise 1 & 0.927 & 0.086 & $-$0.026\,(0.003) & $-$0.045\,(0.002) & $-$0.011\,(0.002) & $-$0.017\,(0.003) & $-$0.003\,(0.003) & $-$0.005\,(0.001) & 0.079 & 0.085 & 1\% \\
Gaussian noise 3 & 0.993$^*$ & 0.204 & $-$0.081\,(0.005) & $-$0.124\,(0.003) & $-$0.080\,(0.004) & $-$0.058\,(0.008) & $-$0.079\,(0.009) & $-$0.058\,(0.002) & 0.136 & 0.203 & 1\% \\
Gaussian noise 5 & 0.997$^*$ & 0.288 & $-$0.135\,(0.006) & $-$0.187\,(0.004) & $-$0.145\,(0.005) & $-$0.082\,(0.012) & $-$0.143\,(0.013) & $-$0.109\,(0.002) & 0.156 & 0.286 & 2\% \\
Defocus blur 1 & 0.529 & 0.049 & $-$0.013\,(0.001) & $-$0.023\,(0.001) & +0.005\,(0.002) & $-$0.018\,(0.001) & $-$0.001\,(0.001) & +0.001\,(0.001) & 0.055 & 0.049 & 0\% \\
Defocus blur 3 & 0.834 & 0.054 & $-$0.014\,(0.001) & $-$0.026\,(0.001) & +0.003\,(0.002) & $-$0.010\,(0.001) & +0.008\,(0.002) & +0.001\,(0.001) & 0.058 & 0.055 & 0\% \\
Defocus blur 5 & 0.977$^*$ & 0.095 & $-$0.022\,(0.003) & $-$0.044\,(0.002) & $-$0.009\,(0.003) & $-$0.010\,(0.004) & $-$0.005\,(0.006) & $-$0.000\,(0.002) & 0.088 & 0.095 & 1\% \\
Fog 1 & 0.598 & 0.049 & $-$0.012\,(0.001) & $-$0.023\,(0.001) & +0.006\,(0.001) & $-$0.017\,(0.001) & $-$0.001\,(0.001) & +0.002\,(0.001) & 0.056 & 0.049 & 0\% \\
Fog 3 & 0.888 & 0.063 & $-$0.015\,(0.001) & $-$0.029\,(0.001) & +0.003\,(0.002) & $-$0.004\,(0.003) & +0.006\,(0.003) & $-$0.000\,(0.001) & 0.068 & 0.063 & 0\% \\
Fog 5 & 0.987$^*$ & 0.175 & $-$0.042\,(0.005) & $-$0.086\,(0.003) & $-$0.043\,(0.004) & $-$0.039\,(0.007) & $-$0.061\,(0.009) & $-$0.037\,(0.002) & 0.139 & 0.175 & 1\% \\
Contrast 1 & 0.684 & 0.050 & $-$0.013\,(0.001) & $-$0.023\,(0.001) & +0.005\,(0.002) & $-$0.016\,(0.001) & $-$0.001\,(0.001) & +0.003\,(0.001) & 0.056 & 0.050 & 0\% \\
Contrast 3 & 0.961$^*$ & 0.072 & $-$0.016\,(0.002) & $-$0.033\,(0.001) & +0.002\,(0.002) & +0.009\,(0.005) & +0.010\,(0.005) & $-$0.001\,(0.001) & 0.076 & 0.071 & 1\% \\
Contrast 5 & 0.999$^*$ & 0.209 & $-$0.015\,(0.008) & $-$0.082\,(0.005) & $-$0.040\,(0.006) & +0.084\,(0.014) & $-$0.024\,(0.021) & $-$0.015\,(0.002) & 0.184 & 0.208 & 2\% \\
JPEG 1 & 0.746 & 0.087 & $-$0.014\,(0.003) & $-$0.038\,(0.002) & +0.003\,(0.003) & $-$0.033\,(0.002) & $-$0.015\,(0.002) & $-$0.001\,(0.001) & 0.094 & 0.086 & 1\% \\
JPEG 3 & 0.844 & 0.121 & $-$0.007\,(0.005) & $-$0.047\,(0.003) & +0.001\,(0.005) & $-$0.047\,(0.003) & $-$0.030\,(0.003) & $-$0.006\,(0.002) & 0.130 & 0.120 & 1\% \\
JPEG 5 & 0.880 & 0.147 & $-$0.001\,(0.006) & $-$0.049\,(0.004) & +0.002\,(0.005) & $-$0.057\,(0.004) & $-$0.041\,(0.005) & $-$0.006\,(0.002) & 0.157 & 0.147 & 1\% \\
\bottomrule
\end{tabular}
}
\end{table}
\begin{table}[htbp]
\centering\scriptsize
\caption{CIFAR-10-C, WideResNet-40-2 (AugMix), $T=2$.}
\resizebox{\linewidth}{!}{\begin{tabular}{lcc|ccc|cc|c|ccc}
\toprule
cell & AUC & truth & \multicolumn{3}{c|}{score weights: gap} & \multicolumn{2}{c|}{emb.\ weights: gap} & ATC-NE & \multicolumn{3}{c}{WSCP (score weights)} \\
 & & & identity & transf. & tight & identity & tight & gap & $\alpha$ & miscov. & $|C|>T$ \\
\midrule
Gaussian noise 1 & 0.927 & 0.029 & +0.010\,(0.001) & $-$0.016\,(0.001) & $-$0.008\,(0.001) & +0.021\,(0.002) & $-$0.001\,(0.002) & $-$0.005\,(0.001) & 0.006 & 0.010 & 20\% \\
Gaussian noise 3 & 0.993$^*$ & 0.093 & $-$0.016\,(0.003) & $-$0.064\,(0.002) & $-$0.053\,(0.003) & +0.019\,(0.005) & $-$0.041\,(0.007) & $-$0.043\,(0.002) & 0.020 & 0.075 & 20\% \\
Gaussian noise 5 & 0.997$^*$ & 0.151 & $-$0.056\,(0.004) & $-$0.115\,(0.002) & $-$0.103\,(0.003) & +0.017\,(0.010) & $-$0.092\,(0.008) & $-$0.088\,(0.002) & 0.028 & 0.144 & 20\% \\
Defocus blur 1 & 0.529 & 0.013 & +0.012\,(0.001) & $-$0.005\,(0.000) & +0.000\,(0.001) & +0.009\,(0.000) & $-$0.002\,(0.000) & +0.001\,(0.001) & 0.004 & 0.003 & 18\% \\
Defocus blur 3 & 0.834 & 0.017 & +0.009\,(0.001) & $-$0.008\,(0.000) & $-$0.004\,(0.001) & +0.013\,(0.001) & $-$0.004\,(0.001) & $-$0.003\,(0.001) & 0.004 & 0.004 & 18\% \\
Defocus blur 5 & 0.977$^*$ & 0.037 & +0.008\,(0.002) & $-$0.022\,(0.001) & $-$0.016\,(0.001) & +0.023\,(0.002) & $-$0.022\,(0.003) & $-$0.011\,(0.001) & 0.010 & 0.015 & 18\% \\
Fog 1 & 0.598 & 0.014 & +0.011\,(0.001) & $-$0.006\,(0.000) & $-$0.000\,(0.001) & +0.009\,(0.000) & $-$0.002\,(0.001) & $-$0.001\,(0.001) & 0.004 & 0.003 & 18\% \\
Fog 3 & 0.888 & 0.019 & +0.012\,(0.001) & $-$0.009\,(0.000) & $-$0.003\,(0.001) & +0.022\,(0.002) & $-$0.004\,(0.001) & $-$0.003\,(0.001) & 0.004 & 0.004 & 19\% \\
Fog 5 & 0.987$^*$ & 0.076 & +0.005\,(0.003) & $-$0.048\,(0.001) & $-$0.037\,(0.002) & +0.023\,(0.005) & $-$0.059\,(0.002) & $-$0.032\,(0.002) & 0.015 & 0.045 & 22\% \\
Contrast 1 & 0.684 & 0.014 & +0.011\,(0.001) & $-$0.005\,(0.000) & $-$0.000\,(0.001) & +0.010\,(0.001) & $-$0.002\,(0.001) & $-$0.000\,(0.001) & 0.004 & 0.003 & 18\% \\
Contrast 3 & 0.961$^*$ & 0.022 & +0.014\,(0.001) & $-$0.010\,(0.001) & $-$0.003\,(0.001) & +0.037\,(0.003) & $-$0.004\,(0.003) & $-$0.002\,(0.001) & 0.005 & 0.006 & 19\% \\
Contrast 5 & 0.999$^*$ & 0.099 & +0.021\,(0.005) & $-$0.054\,(0.003) & $-$0.042\,(0.004) & +0.149\,(0.010) & $-$0.079\,(0.008) & $-$0.022\,(0.003) & 0.031 & 0.070 & 19\% \\
JPEG 1 & 0.746 & 0.028 & +0.019\,(0.002) & $-$0.011\,(0.001) & $-$0.001\,(0.002) & +0.008\,(0.001) & $-$0.009\,(0.001) & $-$0.000\,(0.001) & 0.005 & 0.008 & 20\% \\
JPEG 3 & 0.844 & 0.044 & +0.027\,(0.003) & $-$0.019\,(0.002) & $-$0.006\,(0.003) & +0.004\,(0.002) & $-$0.021\,(0.002) & $-$0.005\,(0.001) & 0.009 & 0.015 & 21\% \\
JPEG 5 & 0.880 & 0.059 & +0.031\,(0.004) & $-$0.026\,(0.002) & $-$0.010\,(0.003) & $-$0.001\,(0.002) & $-$0.033\,(0.002) & $-$0.010\,(0.002) & 0.013 & 0.028 & 21\% \\
\bottomrule
\end{tabular}
}
\end{table}
\begin{table}[htbp]
\centering\scriptsize
\caption{CIFAR-10-C, WideResNet-40-2 (AugMix), $T=3$.}
\resizebox{\linewidth}{!}{\begin{tabular}{lcc|ccc|cc|c|ccc}
\toprule
cell & AUC & truth & \multicolumn{3}{c|}{score weights: gap} & \multicolumn{2}{c|}{emb.\ weights: gap} & ATC-NE & \multicolumn{3}{c}{WSCP (score weights)} \\
 & & & identity & transf. & tight & identity & tight & gap & $\alpha$ & miscov. & $|C|>T$ \\
\midrule
Gaussian noise 1 & 0.927 & 0.012 & +0.021\,(0.001) & $-$0.004\,(0.000) & $-$0.002\,(0.001) & +0.032\,(0.002) & $-$0.002\,(0.002) & $-$0.000\,(0.001) & 0.003 & 0.005 & 24\% \\
Gaussian noise 3 & 0.993$^*$ & 0.051 & +0.011\,(0.002) & $-$0.033\,(0.001) & $-$0.030\,(0.002) & +0.050\,(0.004) & $-$0.035\,(0.004) & $-$0.027\,(0.002) & 0.009 & 0.042 & 25\% \\
Gaussian noise 5 & 0.997$^*$ & 0.091 & $-$0.015\,(0.003) & $-$0.069\,(0.001) & $-$0.067\,(0.002) & +0.063\,(0.008) & $-$0.067\,(0.006) & $-$0.061\,(0.002) & 0.013 & 0.094 & 25\% \\
Defocus blur 1 & 0.529 & 0.005 & +0.016\,(0.001) & $-$0.000\,(0.000) & +0.000\,(0.000) & +0.014\,(0.000) & $-$0.000\,(0.000) & +0.001\,(0.001) & 0.002 & 0.002 & 21\% \\
Defocus blur 3 & 0.834 & 0.007 & +0.016\,(0.001) & $-$0.002\,(0.000) & $-$0.001\,(0.000) & +0.020\,(0.001) & $-$0.002\,(0.000) & $-$0.001\,(0.001) & 0.002 & 0.002 & 22\% \\
Defocus blur 5 & 0.977$^*$ & 0.016 & +0.021\,(0.001) & $-$0.007\,(0.001) & $-$0.007\,(0.001) & +0.038\,(0.002) & $-$0.012\,(0.001) & $-$0.005\,(0.001) & 0.005 & 0.006 & 24\% \\
Fog 1 & 0.598 & 0.005 & +0.017\,(0.000) & $-$0.000\,(0.000) & +0.001\,(0.000) & +0.015\,(0.000) & +0.001\,(0.000) & +0.001\,(0.001) & 0.002 & 0.002 & 21\% \\
Fog 3 & 0.888 & 0.008 & +0.018\,(0.001) & $-$0.003\,(0.000) & $-$0.002\,(0.001) & +0.027\,(0.001) & +0.001\,(0.001) & $-$0.000\,(0.001) & 0.002 & 0.002 & 22\% \\
Fog 5 & 0.987$^*$ & 0.039 & +0.026\,(0.002) & $-$0.024\,(0.001) & $-$0.022\,(0.002) & +0.049\,(0.004) & $-$0.028\,(0.002) & $-$0.019\,(0.002) & 0.007 & 0.022 & 26\% \\
Contrast 1 & 0.684 & 0.005 & +0.017\,(0.001) & $-$0.000\,(0.000) & +0.001\,(0.000) & +0.016\,(0.000) & +0.001\,(0.001) & +0.001\,(0.001) & 0.002 & 0.002 & 21\% \\
Contrast 3 & 0.961$^*$ & 0.009 & +0.021\,(0.001) & $-$0.003\,(0.000) & $-$0.001\,(0.001) & +0.043\,(0.002) & +0.003\,(0.002) & $-$0.000\,(0.001) & 0.002 & 0.003 & 23\% \\
Contrast 5 & 0.999$^*$ & 0.056 & +0.042\,(0.004) & $-$0.034\,(0.002) & $-$0.033\,(0.003) & +0.179\,(0.009) & $-$0.054\,(0.001) & $-$0.016\,(0.003) & 0.014 & 0.035 & 25\% \\
JPEG 1 & 0.746 & 0.011 & +0.028\,(0.002) & $-$0.003\,(0.001) & $-$0.000\,(0.002) & +0.020\,(0.001) & $-$0.004\,(0.001) & +0.002\,(0.001) & 0.003 & 0.004 & 24\% \\
JPEG 3 & 0.844 & 0.021 & +0.037\,(0.003) & $-$0.008\,(0.001) & $-$0.004\,(0.002) & +0.020\,(0.001) & $-$0.012\,(0.001) & $-$0.000\,(0.001) & 0.004 & 0.006 & 25\% \\
JPEG 5 & 0.880 & 0.030 & +0.043\,(0.003) & $-$0.014\,(0.001) & $-$0.011\,(0.002) & +0.019\,(0.002) & $-$0.021\,(0.001) & $-$0.006\,(0.002) & 0.005 & 0.011 & 26\% \\
\bottomrule
\end{tabular}
}
\end{table}

\paragraph{Non-degeneracy fails when the miscoverage is high.} By Remark~\ref{rem:nondeg}, non-degeneracy asks that the
threshold $\tau$ of every missed calibration point be positive and at least $\bar h$, the weighted miss rate times the
weighted mean threshold of the missed points. A missed point with a smaller threshold counts for less than one miss, so the tight
report is then below the weighted miss rate. Table~\ref{tab:nondeg} shows how often this happens on CIFAR-10-C. In our
runs it is rare. With score weights a degenerate fold occurs in \NcifoldTOnedegenerate{} of the \NcifoldTOnedraws{}
EfficientNet-B0 draws at $T=1$ and in no other draw, and changes the mean tight report of a setting by less than
\Nndscoremaxchange{} points, so the tight report is in effect the importance-weighted miss rate of the size-$T$ set.
With embedding weights it is more frequent and changes a setting's mean by less than one point. The last two
columns take the weights out of the question. They give the limit of the tight report when the calibration data come
from the corrupted images themselves, with no shift and no weights. By the proof of Theorem~\ref{thm:corrloo} this limit
is $\E_{\Pt}[m\min\{1,\tau/\E_{\Pt}[h]\}]$, and under the assumptions of that theorem exact weights give the same limit.
It is below the miscoverage in seven of the $15$ EfficientNet-B0 cells at $T=1$. Every cell in which it is below has a
miscoverage of at least \Nndfailmissmin\%, and the shortfall is up to \Nndlimitmaxshort{} points. In each of these
settings the weighted miss rate of our runs is below the miscoverage, so the estimated weights give the missed points
less weight than the corrupted images do. This lowers $\bar h$ and makes degenerate folds rarer in our runs than in this
limit. On Dry Bean a degenerate fold occurs only under
the hardness shift at $\beta=2$ and $T=1$, in $6$ of the $80$ draws with the true weights and in $1$ with score weights,
and changes the tight report by at most \Nnddbmaxchange.

\begin{table}[htbp]
\centering\small
\caption{Non-degeneracy on CIFAR-10-C. Score and embedding weights: number of the $900$ draws of a row ($15$ cells,
$60$ draws each) with a degenerate fold, and the largest change of a setting's mean tight report relative to the
weighted miss rate, in percentage points. Last two columns: the limit of the tight report when the calibration data
come from the corrupted images themselves, without weights, which is also its limit under exact weights; number of the
$15$ cells in which this limit is below the miscoverage, and the largest such shortfall in percentage points.}\label{tab:nondeg}
\begin{tabular}{lccccccc}
\toprule
 & & \multicolumn{2}{c}{score weights} & \multicolumn{2}{c}{embedding weights} & \multicolumn{2}{c}{corrupted calibration data} \\
\cmidrule(lr){3-4}\cmidrule(lr){5-6}\cmidrule(lr){7-8}
model & $T$ & draws & largest change & draws & largest change & cells & largest shortfall \\
\midrule
EfficientNet-B0 & 1 & 33 & 0.01 & 62 & 0.99 & 7 & 14.96 \\
 & 2 & 0 & 0.00 & 13 & 0.00 & 4 & 5.31 \\
 & 3 & 0 & 0.00 & 0 & 0.00 & 3 & 0.83 \\
\addlinespace
WRN-40-2 (AugMix) & 1 & 0 & 0.00 & 26 & 0.02 & 1 & 0.16 \\
 & 2 & 0 & 0.00 & 0 & 0.00 & 0 & 0.00 \\
 & 3 & 0 & 0.00 & 0 & 0.00 & 0 & 0.00 \\
\bottomrule
\end{tabular}

\end{table}

\paragraph{The binned reliability shift.} The clean images are binned by their top-$T$ softmax mass into $20$
equal-count bins, and the corrupted miss rate is predicted as the clean miss rate of each bin averaged over the bins of
the corrupted images. The error of this prediction, which uses the corrupted labels through the truth, estimates the negative of
the binned reliability shift. With both errors divided by the miscoverage, its correlation with the error of the tight
report with score weights is \Nreloldrelcorrnonwall{} (EfficientNet-B0) and \Nrelnewrelcorrnonwall{} (WideResNet)
without walls; with $10$, $20$ or $50$ bins the correlations lie between \Nreldivcorrmin{} and \Nreldivcorrmax{}, and
between \Nreldivcorrnwmin{} and \Nreldivcorrnwmax{} without walls. Without the division they are \NreloldBTwoZerocorr{}
and \NrelnewBTwoZerocorr{}, which partly reflects that larger miscoverages give larger errors, and \Nreloldnonwallcorr{}
and \Nrelnewnonwallcorr{} without walls. The binned prediction lies below the tight report with score weights in
\Nreloldbelow{} and \Nrelnewbelow{} of the \Nreloldnsettings{} settings of each model, by \Nreloldmeandiffpp{} and
\Nrelnewmeandiffpp{} points on average. At the walls the summed absolute binned shift is \Nrelnewwallrelsize{} of the
summed miscoverage for the WideResNet and \Nreloldwallrelsize{} for EfficientNet-B0.

\paragraph{The comparison does not depend on the domain classifier.} Section~\ref{sec:experiments} fits one domain
classifier, gradient boosting, on both the embedding and the sorted scores. To check that this choice does not favor the
sorted scores, we repeated the CIFAR-10-C runs on the same draws with three other domain classifiers
(Table~\ref{tab:robclf}). They are scikit-learn's \texttt{LogisticRegression} and \texttt{MLPClassifier} (one hidden
layer of $100$ units, early stopping), both on standardized features, and the gradient boosting of
Appendix~\ref{app:details} stopped early on a held-out fifth of its training sample. With every classifier, score weights
bring the tight report closer to the miscoverage than embedding weights in all \Ncifmaengroups{} groups of
Table~\ref{tab:main}. Gradient boosting also gives embedding weights their smallest mean absolute gap among the four
classifiers, with and without walls, and the score weights of the main text are closer than the embedding weights of
every classifier in all \Ncifmaengroups{} groups.

\begin{table}[htbp]
\centering\small
\caption{CIFAR-10-C: mean absolute gap of the tight report, in percentage points, with embedding and with score weights
from four domain classifiers fitted on the same draws. Means over the settings without walls or at the walls and their
$60$ draws. Last column: groups of Table~\ref{tab:main} in which score weights have the smaller mean absolute
gap.}\label{tab:robclf}
\begin{tabular}{lccccc}
\toprule
 & \multicolumn{2}{c}{no walls} & \multicolumn{2}{c}{walls} & score weights \\
\cmidrule(lr){2-3}\cmidrule(lr){4-5}
domain classifier & embedding & score & embedding & score & closer in \\
\midrule
Gradient boosting (main text) & 4.20 & 2.93 & 20.85 & 17.84 & 12 of 12 groups \\
Logistic regression & 5.40 & 3.37 & 23.13 & 17.83 & 12 of 12 groups \\
Multilayer perceptron & 4.77 & 3.45 & 22.14 & 17.98 & 12 of 12 groups \\
Gradient boosting, early stopping & 4.27 & 3.53 & 20.94 & 18.89 & 12 of 12 groups \\
\bottomrule
\end{tabular}

\end{table}

\paragraph{A looser weight clip leaves the comparison unchanged.} The clip of $\hat q$ to $[0.01,0.99]$ caps the weights
at $99$. Table~\ref{tab:robclip} gives the reports of the same fitted classifiers under two looser clips. Under every
clip, score weights are closer than embedding weights in all \Ncifmaengroups{} groups, and at the walls the tight report
stays more than $15$ points below the miscoverage on average. The clip does matter for the level of the report with score
weights. In a setting without walls at most \Nrobqabovemax\% of the calibration images have $\hat q>0.99$ with score
weights, but under a looser clip these few images take large weights, so the tight report rises on average and varies
more from draw to draw. Its mean gap moves toward zero while its mean absolute gap grows, so neither clip is better on
both counts. Without walls the report with embedding weights changes by less than $0.1$ points. On Dry Bean under the
hardness shift, the gap with feature weights changes by at most \Ndbclipfeatmaxchange{} and the gap with score weights
rises.

\begin{table}[htbp]
\centering\small
\caption{Sensitivity to the clip of the domain classifier's probability $\hat q$. The classifiers are fitted once and only
the clip changes. CIFAR-10-C: tight report over the settings without walls or at the walls and their $60$ draws; SD is
the standard deviation of the gap over the draws of a setting, averaged over the settings; ``closer in'' counts the groups
of Table~\ref{tab:main} in which score weights have the smaller mean absolute gap. Dry Bean: gap of the tight report as
in Table~\ref{tab:drybean}, means over $80$ draws.}\label{tab:robclip}
\resizebox{\linewidth}{!}{\begin{tabular}{lcccccccccccc}
\toprule
 & \multicolumn{8}{c}{CIFAR-10-C, tight report (percentage points)} & \multicolumn{4}{c}{Dry Bean, hardness, $T=1$} \\
\cmidrule(lr){2-9}\cmidrule(lr){10-13}
 & \multicolumn{3}{c}{score weights, no walls} & \multicolumn{2}{c}{embedding, no walls} & \multicolumn{2}{c}{walls, mean gap} & score & \multicolumn{2}{c}{feature weights} & \multicolumn{2}{c}{score weights} \\
\cmidrule(lr){2-4}\cmidrule(lr){5-6}\cmidrule(lr){7-8}\cmidrule(lr){10-11}\cmidrule(lr){12-13}
clip of $\hat q$ & mean gap & SD & $|\mathrm{gap}|$ & mean gap & $|\mathrm{gap}|$ & score & emb. & closer in & $\beta=1$ & $\beta=2$ & $\beta=1$ & $\beta=2$ \\
\midrule
$[0.01,\,0.99]$ (main text) & $-$1.72 & 2.20 & 2.93 & $-$3.79 & 4.20 & $-$17.6 & $-$19.6 & 12 of 12 & $-$0.081 & $-$0.101 & +0.042 & $-$0.006 \\
$[0.001,\,0.999]$ & $-$0.93 & 3.60 & 3.18 & $-$3.75 & 4.23 & $-$16.9 & $-$18.6 & 12 of 12 & $-$0.081 & $-$0.104 & +0.085 & +0.032 \\
$[10^{-6},\,1-10^{-6}]$ & $-$0.85 & 3.89 & 3.26 & $-$3.75 & 4.23 & $-$16.9 & $-$18.0 & 12 of 12 & $-$0.081 & $-$0.104 & +0.093 & +0.021 \\
\bottomrule
\end{tabular}
}
\end{table}

\paragraph{Fixing the level of WSCP in advance does not remove its excess.} WSCP chooses its level on the unlabeled test
sample of the draw, so its guarantee, which is for a level fixed in advance, does not strictly apply. To check whether
this choice explains why its miscoverage exceeds its level, we reran it with the level chosen in the previous draw,
which is fixed before the current draw (Table~\ref{tab:robwscp}). The excess changes by less than one point in every
row, and WSCP keeps its level in no more settings than before. The data-dependent level is therefore not what makes WSCP
exceed its level. With the true weights on Dry Bean, WSCP keeps its level under the same data-dependent choice
(Section~\ref{sec:experiments}), which points to the estimated weights.

\begin{table}[htbp]
\centering\small
\caption{WSCP with its level chosen as in the main text and with the level chosen in the previous draw. CIFAR-10-C: score
weights, settings without walls. Dry Bean: feature weights. Valid: settings in which the mean miscoverage is at most the
mean level. Excess: miscoverage minus level in percentage points, averaged over the settings of a row and their draws
($60$ on CIFAR-10-C, $80$ on Dry Bean; the first draw of a setting has no previous level).}\label{tab:robwscp}
\begin{tabular}{lcccccc}
\toprule
 & & & \multicolumn{2}{c}{level of the main text} & \multicolumn{2}{c}{level of the previous draw} \\
\cmidrule(lr){4-5}\cmidrule(lr){6-7}
 & $T$ & settings & valid & excess & valid & excess \\
\midrule
CIFAR-10-C, EfficientNet-B0 & 1 & 12 & 7 & +2.24 & 7 & +2.16 \\
 & 2 & 12 & 1 & +5.29 & 1 & +5.35 \\
 & 3 & 12 & 2 & +3.68 & 1 & +3.71 \\
\addlinespace
CIFAR-10-C, WRN-40-2 (AugMix) & 1 & 9 & 8 & $-$0.54 & 8 & $-$0.37 \\
 & 2 & 9 & 4 & +0.30 & 3 & +0.35 \\
 & 3 & 9 & 4 & +0.10 & 3 & +0.15 \\
\addlinespace
Dry Bean, hardness, $\beta=1$ & 1 & 1 & 0 & +9.67 & 0 & +10.29 \\
Dry Bean, hardness, $\beta=2$ & 1 & 1 & 0 & +16.41 & 0 & +17.25 \\
\bottomrule
\end{tabular}

\end{table}

\paragraph{A natural shift shows the same pattern.} CIFAR-10.1 was collected to match CIFAR-10, yet the miscoverage of
both models at least doubles on it (Table~\ref{tab:cifar101}). Without weights the tight report stays at the clean miss
rate and misses the whole rise. With each of the four domain classifiers of Table~\ref{tab:robclf}, score weights have a
smaller mean absolute gap than embedding weights and than no weights in all six settings, and the identity report with
score weights stays above the miscoverage at $T=2,3$. With logistic regression, the perceptron and early-stopped gradient
boosting, the tight report still understates, and its mean gap is within \Nnatrelmatch{} points of that of the binned
prediction, computed as above with CIFAR-10.1 in place of the corrupted images. That gap estimates the negative of the
reliability shift, the bias that score weights leave by Proposition~\ref{prop:decomp}. At the same top-$T$ mass the set
misses more often on CIFAR-10.1 than on clean images. The gradient boosting of the main text gives mean gaps closer to
zero. Since the reliability shift does not depend on the domain classifier, this difference is the other term of
Proposition~\ref{prop:decomp}, the error of its weights on the missed points. With it the mean absolute gap of the tight
report is the largest of the four in three of the six settings. At $T=1$ the identity report fails in seven of the eight combinations of model
and classifier, and ATC-NE has a smaller mean absolute gap than the tight report with any of the score weights in five
of the six settings.

\begin{table}[htbp]
\centering\small
\caption{CIFAR-10.1 as the test distribution, with calibration on clean CIFAR-10 test images. Miscoverage: miss rate of
the size-$T$ set over the $2{,}000$ CIFAR-10.1 images. Other columns: gap (report $-$ miscoverage) in percentage points,
means over $60$ draws. The tight report is given without weights, with embedding and score weights from the gradient
boosting of the main text, and with score weights from logistic regression (LR), the perceptron (MLP) and early-stopped
gradient boosting (GB-ES). The binned prediction uses all images and has no draws. The identity report uses the score
weights of the main text. The largest standard error over draws is \Nnatmaxse.}\label{tab:cifar101}
\resizebox{\linewidth}{!}{\begin{tabular}{lccccccccccc}
\toprule
 & & & \multicolumn{7}{c}{tight report} & & identity \\
\cmidrule(lr){4-10}
 & & & & \multicolumn{2}{c}{gradient boosting} & \multicolumn{3}{c}{score weights from} & binned & & report \\
\cmidrule(lr){5-6}\cmidrule(lr){7-9}
model & $T$ & miscoverage & no weights & embedding & score & LR & MLP & GB-ES & prediction & ATC-NE & (score) \\
\midrule
EfficientNet-B0 & 1 & 18.05\% & $-$8.99 & $-$5.02 & $-$0.68 & $-$3.40 & $-$3.35 & $-$3.73 & $-$3.55 & $-$0.66 & +0.90 \\
 & 2 & 7.00\% & $-$3.85 & $-$2.03 & +0.77 & $-$1.57 & $-$1.38 & $-$1.26 & $-$1.53 & +1.05 & +4.21 \\
 & 3 & 4.00\% & $-$2.59 & $-$1.65 & +0.02 & $-$1.52 & $-$1.49 & $-$1.25 & $-$1.46 & +0.56 & +4.78 \\
\addlinespace
WRN-40-2 (AugMix) & 1 & 11.30\% & $-$6.55 & $-$3.26 & $-$0.57 & $-$2.21 & $-$2.71 & $-$2.85 & $-$2.55 & $-$1.21 & $-$1.61 \\
 & 2 & 3.65\% & $-$2.46 & $-$1.45 & $-$0.30 & $-$1.09 & $-$1.14 & $-$1.26 & $-$1.22 & +0.34 & +2.44 \\
 & 3 & 1.65\% & $-$1.23 & $-$0.90 & $-$0.61 & $-$0.69 & $-$0.74 & $-$0.80 & $-$0.66 & +0.10 & +3.30 \\
\bottomrule
\end{tabular}
}
\end{table}

\section{The identity report at \texorpdfstring{$T=1$}{T=1}}\label{app:tone}

Table~\ref{tab:noshift} gives a no-shift control: $60$ draws (seed $0$) of $1{,}000$ calibration images from the
$10{,}000$ clean CIFAR-10 test images, the other $9{,}000$ images as test points, and no weights. It gives the exact
miscoverage of the size-$T$ set on the test points; the gap of the identity report with our clipped score and with the
unclipped cross-entropy of \citet{gauthier2025backward}, each against the miscoverage under the same score; and, for the
clipped score, the gap of the tight report and the terms of
\[
P(Y\notin\hat C)=\E[\tilde\alpha]\,\E\bigl[\Ind\{Y\notin\hat C\}/\tilde\alpha\bigr]
+\mathrm{Cov}\bigl(\tilde\alpha,\Ind\{Y\notin\hat C\}/\tilde\alpha\bigr)
\]
over the test points of all draws; the last column gives the covariance with the unclipped score. With the clipped score the level is below one at every test point, so $\hat C$ is
the size-$T$ set. With the clipped score the covariance is positive at every $T$, so the condition of
Theorem~\ref{thm:coverage} fails, and the
identity report stays above the miscoverage at $T=2,3$ because the middle factor is small. At $T=1$ the covariance is
positive without the clip as well. Since
$\Ind\{Y\notin\hat C\}/\tilde\alpha\le E(Y)\Ind\{Y\notin\hat C\}$, that factor is at most
$\E[E(Y)]-\E[E(Y)\Ind\{Y\in\hat C\}]$, so it is small when covered points carry much of the e-value's expectation. They
carry \NnosholdTOnecovered\% (EfficientNet-B0) and \NnoshnewTOnecovered\% (WideResNet) of it at $T=1$, where a covered
true label is the model's top prediction and its score is near zero, and at least \Nnoshcoveredmintwothree\% at $T=2,3$.
At $T=1$ the sign of the gap therefore depends on the scale of the levels, which is roughly proportional to the mean
calibration score. The unclipped cross-entropy of EfficientNet-B0 has a few extreme values that raise $\E[\tilde\alpha]$
at $T=1$ by \Nnoshclipalphashift{} points; the WideResNet has no score above the clip. For EfficientNet-B0 the clip ties scores at the top-$T$ boundary for \Ntiesold{} of
the \Ntiespairsold{} pairs of image and $T\in\{1,2,3\}$, which our results allow.

\begin{table}[htbp]
\centering\small
\caption{No shift: clean CIFAR-10, no weights, $n=1{,}000$, means over $60$ draws. Truth: miscoverage of the size-$T$
set with the clipped score. Gaps (report $-$ miscoverage under the same score) and covariances in percentage points.}\label{tab:noshift}
\resizebox{\linewidth}{!}{\begin{tabular}{llcccccccc}
\toprule
 & & & \multicolumn{2}{c}{identity gap} & & & & \multicolumn{2}{c}{covariance} \\
\cmidrule(lr){4-5}\cmidrule(lr){9-10}
model & $T$ & truth & clipped & unclipped & tight gap & $\E[\tilde\alpha]$ & $\E[\Ind\{Y\notin\hat C\}/\tilde\alpha]$ & clipped & unclipped \\
\midrule
EfficientNet-B0 & 1 & 8.91\% & $-$1.29 & +1.05 & $-$0.03 & 7.66\% & 0.67 & +3.80 & +3.19 \\
 & 2 & 3.05\% & +1.59 & +3.02 & +0.03 & 4.63\% & 0.43 & +1.06 & $-$0.14 \\
 & 3 & 1.40\% & +2.41 & +3.67 & $-$0.01 & 3.81\% & 0.26 & +0.41 & +0.52 \\
\midrule
WRN-40-2 (AugMix) & 1 & 4.93\% & $-$1.73 & $-$1.73 & $-$0.09 & 3.21\% & 0.71 & +2.65 & +2.65 \\
 & 2 & 1.24\% & +0.95 & +0.95 & 0.00 & 2.19\% & 0.30 & +0.59 & +0.59 \\
 & 3 & 0.48\% & +1.43 & +1.43 & $-$0.05 & 1.91\% & 0.14 & +0.21 & +0.21 \\
\bottomrule
\end{tabular}
}
\end{table}

\paragraph{The code of \citet{gauthier2025backward}.} In their notebook, $\tilde\alpha$ is found by bisection on $(0,1)$
with tolerance $0.005$, and coverage is counted on the set computed at the last level the bisection tries, which can lie
on the side where the set has more than $T$ labels. We replayed this search with their score on the draws above
($n=1{,}000$ and all $9{,}000$ test points of each draw; their runs use one test point per run). At $T=1$ that set has
more than one label for \Nbcpshareover\% of the test points (mean size \Nbcpmeansize), and its miscoverage is
\Nbcpmisnb\%, against \Nbcpmisexact\% for the size-$1$ set; $\E[\tilde\alpha]$ is \Nbcpealpha\%, above both.

\paragraph{The floor on the score binds only for EfficientNet-B0.} No probability of the WideResNet or of the Dry Bean
classifier is below $10^{-12}$, so their scores are never clipped. The identity report of the WideResNet still fails at
$T=1$, without shift (Table~\ref{tab:noshift}) and in all \NcifnewTOnenonwallncells{} of its settings without walls
under shift, so this failure does not come from the floor. For EfficientNet-B0 the floor does lower the identity report. It caps \Nunclipbound{} of the $10{,}000$ clean true-label
scores, which lowers the mean calibration score and with it every level. Table~\ref{tab:unclip} compares the two scores
under shift with the same score weights. Without the floor the identity report is valid in two more settings at $T=1$
and one more at $T=2$, all without walls, and its validity at the walls does not change. Under non-degeneracy the tight
report depends on the scores only through the size-$T$ set, and its mean absolute gap changes by at most
\Nuncliptightrowmax{} points in any row.

\begin{table}[htbp]
\centering\small
\caption{EfficientNet-B0 on CIFAR-10-C with the floor $10^{-12}$ on the probabilities in the score and without it, with
the same score weights. Each report is compared with the miscoverage under the same score. Gaps in percentage points,
averaged over the settings of a row and their $60$ draws.}\label{tab:unclip}
\begin{tabular}{lcccccccc}
\toprule
 & & & \multicolumn{4}{c}{identity report} & \multicolumn{2}{c}{tight report} \\
\cmidrule(lr){4-7}\cmidrule(lr){8-9}
 & & & \multicolumn{2}{c}{valid settings} & \multicolumn{2}{c}{mean gap} & \multicolumn{2}{c}{mean $|\mathrm{gap}|$} \\
\cmidrule(lr){4-5}\cmidrule(lr){6-7}\cmidrule(lr){8-9}
 & $T$ & settings & floor & no floor & floor & no floor & floor & no floor \\
\midrule
no walls & 1 & 12 & 9 & 11 & +2.55 & +4.51 & 4.74 & 4.74 \\
 & 2 & 12 & 10 & 11 & +2.71 & +3.91 & 4.78 & 4.78 \\
 & 3 & 12 & 11 & 11 & +4.29 & +5.27 & 3.48 & 3.49 \\
\addlinespace
walls & 1 & 3 & 0 & 0 & $-$22.06 & $-$22.01 & 42.98 & 42.98 \\
 & 2 & 3 & 0 & 0 & $-$32.96 & $-$32.93 & 48.67 & 48.67 \\
 & 3 & 3 & 0 & 0 & $-$27.78 & $-$27.76 & 42.39 & 42.39 \\
\bottomrule
\end{tabular}

\end{table}

\section{Calibrating the domain classifier}\label{app:calibration}

Table~\ref{tab:calibration} gives, for each Dry Bean shift and each way of fitting the weights, the ECE and log-loss of
the domain classifier on held-out data, the Spearman correlation between the estimated and the true weights on the
calibration beans, the mean squared error of the normalized weights relative to the true weights (MSE), and the gaps of
the identity report ($T=1$) and the tight report ($T=1,2$). The four calibration methods are, in the order of the table,
Platt scaling with one model, Platt scaling with five models, isotonic regression with five models, and early stopping
on held-out log-loss. Under the hardness shift at $\beta=1$ they lower the MSE from \NdbhardOnemsegb{} to between
\NdbhardOnemsecalmin{} and \NdbhardOnemsecalmax, and they lower it in \Ndbcalmselower{} of the \Ndbcalmsetotal{}
combinations of method and setting.

\begin{table}[htbp]
\centering\scriptsize
\caption{Dry Bean: calibrated domain classifiers. Means over $80$ draws; the largest standard error of a gap is
\Ndbcalmaxse.}\label{tab:calibration}
\resizebox{\linewidth}{!}{\begin{tabular}{llccccccc}
\toprule
shift & weights & ECE & log-loss & rank corr. & MSE & identity gap ($T{=}1$) & tight gap ($T{=}1$) & tight gap ($T{=}2$) \\
\midrule
hardness $\beta=1$ & features & 0.054 & 0.576 & 0.68 & 3.17 & +0.013 & $-$0.081 & $-$0.027 \\
hardness $\beta=1$ & + Platt (1 model) & 0.041 & 0.565 & 0.68 & 2.27 & $-$0.064 & $-$0.121 & $-$0.037 \\
hardness $\beta=1$ & + Platt (5 models) & 0.046 & 0.561 & 0.71 & 2.19 & $-$0.065 & $-$0.121 & $-$0.037 \\
hardness $\beta=1$ & + isotonic (5 models) & 0.038 & 0.561 & 0.70 & 2.81 & $-$0.030 & $-$0.104 & $-$0.031 \\
hardness $\beta=1$ & features, early stopping & 0.043 & 0.560 & 0.73 & 2.20 & $-$0.042 & $-$0.110 & $-$0.036 \\
hardness $\beta=1$ & logistic regression & 0.061 & 0.600 & 0.58 & 2.82 & $-$0.127 & $-$0.165 & $-$0.046 \\
hardness $\beta=1$ & sorted scores & 0.056 & 0.546 & 0.80 & 5.36 & +0.267 & +0.042 & +0.015 \\
\midrule
hardness $\beta=2$ & features & 0.049 & 0.364 & 0.76 & 21.29 & +0.106 & $-$0.101 & $-$0.061 \\
hardness $\beta=2$ & + Platt (1 model) & 0.039 & 0.357 & 0.76 & 20.14 & +0.046 & $-$0.129 & $-$0.074 \\
hardness $\beta=2$ & + Platt (5 models) & 0.046 & 0.354 & 0.77 & 19.03 & +0.027 & $-$0.139 & $-$0.082 \\
hardness $\beta=2$ & + isotonic (5 models) & 0.041 & 0.357 & 0.76 & 20.47 & +0.077 & $-$0.111 & $-$0.070 \\
hardness $\beta=2$ & features, early stopping & 0.051 & 0.364 & 0.75 & 20.07 & +0.024 & $-$0.145 & $-$0.080 \\
hardness $\beta=2$ & logistic regression & 0.071 & 0.523 & 0.56 & 25.81 & $-$0.245 & $-$0.292 & $-$0.122 \\
hardness $\beta=2$ & sorted scores & 0.043 & 0.301 & 0.86 & 15.81 & +0.340 & $-$0.006 & $-$0.036 \\
\midrule
easy $\beta=1$ & features & 0.084 & 0.649 & 0.61 & 2.05 & +0.028 & +0.010 & +0.002 \\
easy $\beta=1$ & + Platt (1 model) & 0.030 & 0.621 & 0.61 & 0.28 & +0.059 & +0.041 & +0.007 \\
easy $\beta=1$ & + Platt (5 models) & 0.032 & 0.620 & 0.64 & 0.26 & +0.059 & +0.042 & +0.007 \\
easy $\beta=1$ & + isotonic (5 models) & 0.030 & 0.620 & 0.63 & 0.41 & +0.049 & +0.032 & +0.006 \\
easy $\beta=1$ & features, early stopping & 0.040 & 0.619 & 0.67 & 0.35 & +0.048 & +0.032 & +0.006 \\
easy $\beta=1$ & logistic regression & 0.037 & 0.624 & 0.59 & 0.37 & +0.059 & +0.043 & +0.007 \\
easy $\beta=1$ & sorted scores & 0.080 & 0.625 & 0.77 & 1.47 & +0.014 & $-$0.003 & $-$0.001 \\
\midrule
easy $\beta=2$ & features & 0.069 & 0.590 & 0.74 & 2.82 & +0.027 & +0.009 & +0.002 \\
easy $\beta=2$ & + Platt (1 model) & 0.048 & 0.576 & 0.74 & 0.70 & +0.042 & +0.027 & +0.005 \\
easy $\beta=2$ & + Platt (5 models) & 0.051 & 0.573 & 0.76 & 0.53 & +0.043 & +0.029 & +0.005 \\
easy $\beta=2$ & + isotonic (5 models) & 0.037 & 0.569 & 0.76 & 0.92 & +0.037 & +0.021 & +0.004 \\
easy $\beta=2$ & features, early stopping & 0.049 & 0.570 & 0.76 & 0.62 & +0.038 & +0.025 & +0.004 \\
easy $\beta=2$ & logistic regression & 0.060 & 0.598 & 0.61 & 1.15 & +0.058 & +0.042 & +0.007 \\
easy $\beta=2$ & sorted scores & 0.064 & 0.551 & 0.89 & 2.22 & +0.017 & $-$0.001 & $-$0.001 \\
\midrule
benign $\beta=1$ & features & 0.073 & 0.608 & 0.80 & 3.30 & +0.009 & $-$0.011 & $-$0.001 \\
benign $\beta=1$ & + Platt (1 model) & 0.030 & 0.584 & 0.80 & 0.61 & +0.025 & +0.009 & +0.001 \\
benign $\beta=1$ & + Platt (5 models) & 0.033 & 0.580 & 0.83 & 0.50 & +0.025 & +0.009 & +0.002 \\
benign $\beta=1$ & + isotonic (5 models) & 0.034 & 0.584 & 0.82 & 1.45 & +0.021 & +0.004 & +0.001 \\
benign $\beta=1$ & features, early stopping & 0.034 & 0.565 & 0.93 & 0.35 & +0.017 & +0.003 & 0.000 \\
benign $\beta=1$ & logistic regression & 0.032 & 0.556 & 0.98 & 0.09 & +0.012 & $-$0.001 & 0.000 \\
benign $\beta=1$ & sorted scores & 0.073 & 0.623 & 0.61 & 1.81 & +0.016 & +0.001 & $-$0.001 \\
\midrule
benign $\beta=2$ & features & 0.048 & 0.442 & 0.91 & 6.22 & +0.025 & $-$0.007 & 0.000 \\
benign $\beta=2$ & + Platt (1 model) & 0.035 & 0.429 & 0.91 & 4.76 & +0.034 & +0.011 & +0.002 \\
benign $\beta=2$ & + Platt (5 models) & 0.039 & 0.422 & 0.93 & 4.26 & +0.034 & +0.015 & +0.003 \\
benign $\beta=2$ & + isotonic (5 models) & 0.035 & 0.427 & 0.92 & 5.17 & +0.032 & +0.009 & +0.002 \\
benign $\beta=2$ & features, early stopping & 0.035 & 0.403 & 0.96 & 3.31 & +0.024 & +0.006 & +0.001 \\
benign $\beta=2$ & logistic regression & 0.029 & 0.392 & 0.99 & 0.58 & +0.023 & +0.002 & +0.001 \\
benign $\beta=2$ & sorted scores & 0.053 & 0.495 & 0.69 & 8.93 & +0.021 & 0.000 & $-$0.001 \\
\bottomrule
\end{tabular}
}
\end{table}

\end{document}